\documentclass[final,5p,times,twocolumn]{elsarticle}

\usepackage{amsmath,amsfonts,dsfont}
\usepackage{amssymb}
\usepackage{amsthm}
\usepackage{xcolor}
\usepackage[noend]{algorithmic}
\usepackage{algorithm}
\usepackage{multirow}

\DeclareMathOperator*{\argmin}{argmin} 
\DeclareFontFamily{U}{mathx}{}
\DeclareFontShape{U}{mathx}{m}{n}{<-> mathx10}{}
\DeclareSymbolFont{mathx}{U}{mathx}{m}{n}
\DeclareMathAccent{\widehat}{0}{mathx}{"70}
\DeclareMathAccent{\widecheck}{0}{mathx}{"71}

\newtheorem{theorem}{Theorem}
\newtheorem{definition}{Definition}
\newtheorem{remark}{Remark}
\newtheorem{proposition}{Proposition}
\newtheorem{assumption}{Assumption}

\journal{Robotics and Autonomous Systems}

\begin{document}

\begin{frontmatter}



\title{Pro-Routing: Proactive Routing of Autonomous Multi-Capacity Robots for Pickup-and-Delivery Tasks}


\author[1]{Daniel Garces\corref{cor1}}
\ead{dgarces@g.harvard.edu}

\author[1]{Stephanie Gil}
\ead{sgil@seas.harvard.edu}

\cortext[cor1]{Corresponding author}

\affiliation[1]{organization={John A. Paulson School Of Engineering And Applied Sciences, Harvard University},
            addressline={150 Western Ave}, 
            city={Boston},
            postcode={02134}, 
            state={MA},
            country={USA}}

\begin{abstract}
We consider a multi-robot setting, where we have a fleet of multi-capacity autonomous robots that must service spatially distributed pickup-and-delivery requests with fixed maximum wait times. Requests can be either scheduled ahead of time or they can enter the system in real-time. In this setting, stability for a routing policy is defined as the cost of the policy being uniformly bounded over time. Most previous work either solve the problem offline to theoretically maintain stability or they consider dynamically arriving requests at the expense of the theoretical guarantees on stability. In this paper, we aim to bridge this gap by proposing a novel proactive rollout-based routing framework that adapts to real-time demand while still provably maintaining the stability of the learned routing policy. We derive provable stability guarantees for our method by proposing a fleet sizing algorithm that obtains a sufficiently large fleet that ensures stability by construction. To validate our theoretical results, we consider a case study on real ride requests for Harvard’s Evening Van System. We also evaluate the performance of our framework using the currently deployed smaller fleet size. In this smaller setup, we compare against the currently deployed routing algorithm, greedy heuristics, and Monte-Carlo-Tree-Search-based algorithms. Our empirical results show that our framework maintains stability when we use the sufficiently large fleet size found in our theoretical results. For the smaller currently deployed fleet size, our method services $6\%$ more requests than the closest baseline while reducing median passenger wait times by $33\%$.
\end{abstract}




\end{frontmatter}


\section{Introduction}
The development of autonomous robotic platforms has led to the emergence of cooperative multi-robot systems, where robots coordinate to solve complex tasks more efficiently. These multi-robot systems have been successfully deployed in different applications, including unmanned aerial and ground vehicle routing for surveillance and maintenance tasks \cite{Nishi2005TRO, Zheng2005TRO, Bopardikar2014TRO, Testa2022TRO}, routing of robot fleets for warehouse operations \cite{Sorbelli2022TRO}, multi-robot task assignment for musical performances \cite{Chopra2017TRO}, multi-robot pickup and delivery tasks with aerial and ground robots \cite{Camisa2023TRO}, autonomous taxicab routing \cite{Kondor2022, Garces2023ICRA, Garces2024ICRA, Garces2024AAMAS}, UAV resupply scheduling \cite{Arribas2023TRO}, and multi-robot service coordination for pandemic responses \cite{Fu2023TRO}. Although all these works have shown very promising results for the development of cooperative multi-robot systems, most of them have focused on single-capacity-robot settings, where robots can only tackle one request at a time. For this reason, there remains a need to generalize these cooperative algorithms to multi-capacity-robot settings, where each robot can handle multiple requests simultaneously. Such settings are common in transportation and delivery applications, including autonomous public transportation \cite{Jokinen2011}, autonomous ride-sharing \cite{Ongel2019, SHAHEEN2019}, freight management and optimization \cite{Xidias2022}, autonomous package deliveries \cite{Matthew2015, Reed2022, Lee2023}, and warehouse package redistribution \cite{Ham01022021}. 

In this paper, we focus on multi-capacity-robot settings, where we have a fleet of multi-capacity autonomous robots that must service spatially distributed pickup-and-delivery requests with fixed maximum wait times. The pickup-and-delivery requests have mixed lead-times, which means that requests can be either scheduled ahead of time or enter the system in real-time. Under this setup, the location and number of future requests is unknown a-priori. The goal of our system is then to find robot routes that minimize request wait times for current and future requests, while satisfying the maximum wait times and robot capacity constraints. More specifically, we seek a routing policy that adapts to real-time demand while still guaranteeing stability for the system, where stability is defined as the cost of the policy being uniformly bounded over time. In this setup, the cost is described as the sum of wait times for the requests, and hence stability implies that no request waits indefinitely.

Since generating routes for a fleet of robots is a combinatorial problem that grows in complexity with the number of robots and the number of requests in the system, finding an optimal solution for this problem is intractable in a large setting \cite{Garey1990, Toth2002}. 
Multiple works in the literature have tried to find reasonable sub-optimal solutions for the multi-capacity robot routing problem by considering offline setups where requests are all scheduled ahead of time \cite{LYU2019, Zuo2021, WANG2023IC, Huang2020DR, Ma2021, Chen2021, Song2023}. These offline algorithms are able to leverage principled assignment methods to theoretically guarantee the stability of the resulting routing policies. However, the assumption that all requests are known a-piori prevents these methods from generalizing to real-time setups where requests arrive dynamically and hence routes have to be adapted online. Other works have tried to tackle this real-time dynamic setup by considering insertion heuristics \cite{Gomes2014, NARAYAN2017, AlonsoMora2017, Sun2021} and dynamic optimization approaches \cite{Huang2020CBO, BRUNI2014, Rongge2023, VANENGELEN2018, Wilbur2022}, but often times at the expense of provable stability for the resulting routing policies.

\begin{figure*}
    \centering
    \includegraphics[width=\linewidth]{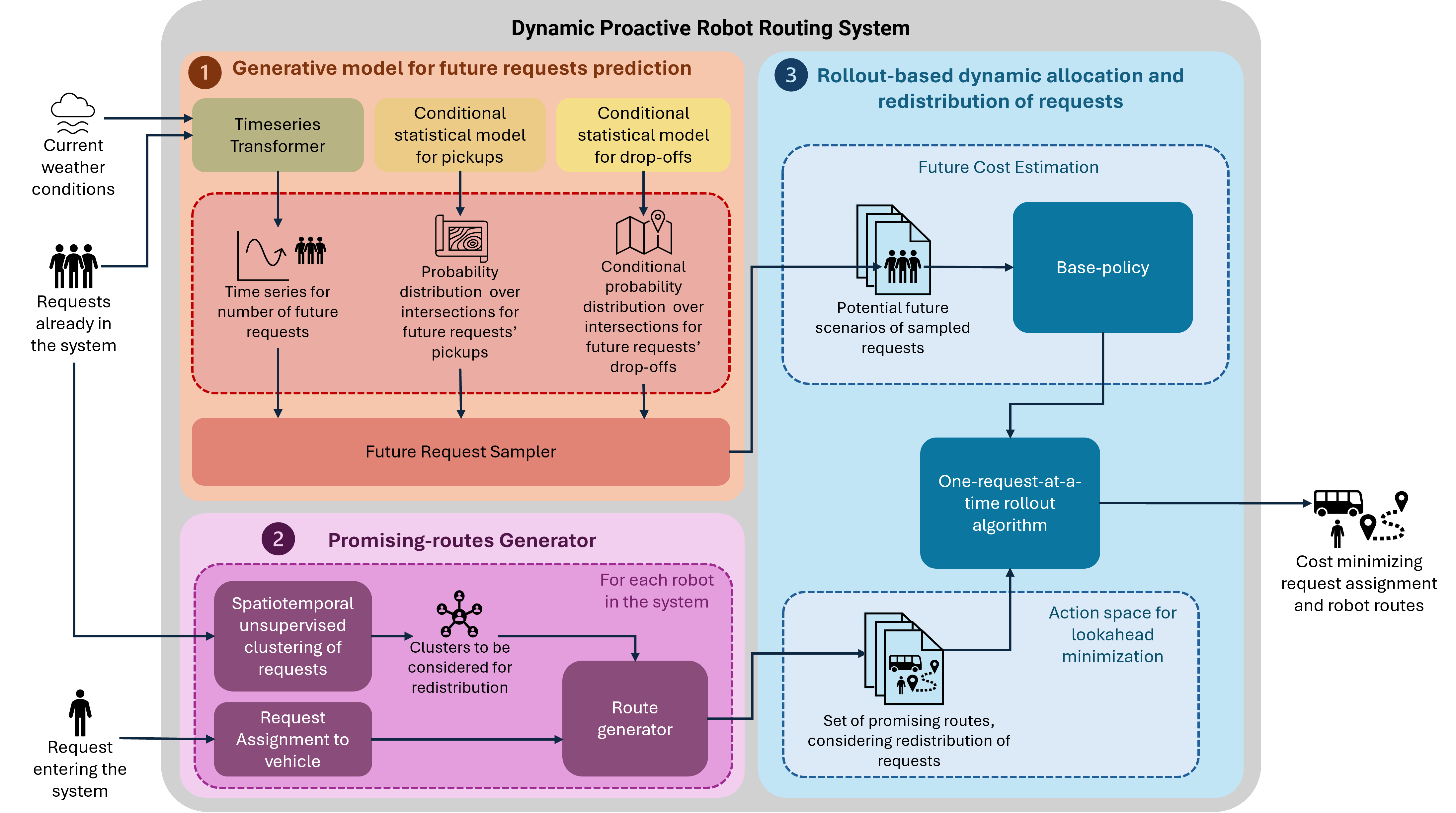}
    \vspace{-5pt}
    \caption{\small{General overview of our proposed method with inputs and outputs. Our method is composed of three main modules: 1) Generative model for future request prediction; 2) Promising-routes generator; 3) Rollout-based dynamic allocation and redistribution of requests.}}
    \label{fig:general_overview}
\end{figure*}

In this paper, we leverage recent advances in model-based Reinforcement Learning (RL) \cite{bertsekas2019reinforcement, bertsekas2020rollout, Bertsekas2022AlphaZero, Bertsekas2023RLCourse},  time-series prediction \cite{wu2023timesnet, nie2023a, wang2024timemixer, liu2024itransformer, wang2024tssurvey}, and unsupervised clustering \cite{Campello2015, Malzer2020}, to create a novel proactive rollout-based routing framework with future request planning that allows for real-time adaptation while provably guaranteeing stability of the learned routing policy.
An overview of our method is shown in Fig.~\ref{fig:general_overview}. 
Our method relies primarily on a new rollout algorithm that we propose called one-request-at-a-time rollout (see module 3 in Fig.~\ref{fig:general_overview}), where we assign one request at a time to a robot in the fleet, creating routes through which robots service multiple requests simultaneously. By assigning requests one at a time, our algorithm is able to remove the combinatorial dependence of the action space on the number of requests, while maintaining the standard rollout decision-making structure. 
The standard rollout decision-making structure \cite{bertsekas2020rollout, Bertsekas2023RLCourse} is composed of three main components: 1) a one-step look-ahead minimization where the cost of immediate controls are estimated using Monte-Carlo Simulations; 2) a future cost approximation that estimates the future cost of each immediate control by executing a simple-to-compute heuristic, known as the base policy, for a finite truncated time horizon; and 3) a terminal cost approximation that compensates for the truncated execution of the base policy. 

Utilizing this decision-making structure, we modify the one-step look-ahead minimization and future cost estimation to allow our method to better adapt to real-time demand. In other words, the main algorithmic novelty of our approach lies in the real-time demand adaptation mechanism that feeds potential actions and future costs estimations to our proposed one-request-at-a-time algorithm to create robot routes that better accommodate potential future requests. The real-time demand adaptation mechanism is composed of two main components: 1) a promising-routes generator (see module 2 in Fig.~\ref{fig:general_overview}) that uses unsupervised clustering of requests to determine feasible request reassignments across robots and creates a set of feasible immediate controls or routes to be considered in the look-ahead minimization of rollout; 2) a generative demand model (see module 1 in Fig.~\ref{fig:general_overview}) that uses a time-series transformer and empirical conditional distributions to sample potential sequences of future requests that are then used to estimate future costs of immediate controls.

To provably guarantee that the learned routing policy obtained by our framework is stable, we build on top off recent theoretical advances \cite{Bertsekas2022AlphaZero, Bertsekas2023RLCourse} that have shown that rollout algorithms with a stable base policy will learn a stable near-optimal policy.
In this paper, we hence focus on how to characterize stability for the multi-capacity robot routing setting and how to find a stable base policy. We first propose a cost function formulated as the sum of wait times for serviced requests. This cost function becomes unbounded if a request is not serviced. Using this cost function, we then propose a greedy base policy with computation times that are amenable to the real-time nature of the problem. We show that our base policy is stable if and only if there is a sufficiently large fleet of robots that will allow the proposed greedy policy to service all requests in a finite time horizon almost surely. Using this result, we design an algorithm to find a sufficiently large fleet size that obtains stability for the proposed base policy. We first consider iterative fleet-sizing algorithms commonly used in the literature. We show that these algorithms are not guaranteed to find a sufficiently large fleet size for stability of our proposed greedy base policy. We then propose an improved fleet sizing algorithm that addresses the limitations of iterative fleet-sizing algorithms, and we show that our fleet-sizing algorithm guarantees stability of our proposed base policy by construction.

We validate our theoretical results and verify the performance of our proposed method by considering a case study on real request data from Harvard's Evening Van Service, a university-wide minibus transportation service that allows students to book rides between any two street intersections inside the van's operating area. Requests for this service can be either booked in advance, or placed in real-time for immediate pickup. Using the real historical ride requests, we empirically demonstrate that our fleet-sizing algorithm obtains a sufficiently large fleet size that allows our proposed greedy base policy and our one-request-at-a-time rollout algorithm to service all new requests within a held-out test set, maintaining stability in practice. 

We also verify the performance of our routing framework using the fleet size currently deployed in the real Evening Van Service. Even though this fleet size is smaller than the one found by our proposed fleet-sizing algorithm, we consider this scenario to evaluate the performance of our proposed framework using the currently available resources for the Evening Van Service. Under this setup, we run a comparison study where we consider several baselines, including multiple Monte-Carlo Tree Search (MCTS) proactive routing frameworks inspired by \cite{Wilbur2022}, different greedy policies using our proposed greedy base policy with different costs, and the currently deployed evening van planning framework. Based on this comparison study, we empirically demonstrate that our method is able to service around $6\%$ more requests requests than the closest baseline during the same $20$ testing days while reducing median passenger wait times at the station to around $2$ minutes compared to $3$ to $10$ minutes of median passenger wait times for all the other baselines. In terms of runtime, our method is able to generate a route for a request in less than $20$ seconds, being around $8$ times faster than the other baselines that also consider future requests. Additionally, we perform different ablations studies to illustrate the advantages provided by our proposed generative model for future request prediction, and our proposed promising-routes generator.

\section{Related Works}
In this section, we review relevant works from the literature in three main areas: multi-capacity robot routing, fleet sizing algorithms, and demand prediction. We divide related works into these three areas to provide a more comprehensive review of related methods and better compare them to the different components of our proposed approach.

\subsection{Multi-capacity Robot Routing}
The multi-capacity robot routing problem corresponds to a special case of the vehicle routing problem \cite{Toth2002}, where robots in the system can service more than one request at a time. Since even the vehicle routing problem is NP-hard \cite{Garey1990, Toth2002}, most work in the literature has focused on developing approximation algorithms to obtain near-optimal routes. Some works have considered setups where requests enter the system before the robots are dispatched. Under these setups, requests are hence known a-priori and routes can be optimized using the exact request locations. To do this, different authors have considered methods such as clustering heuristics \cite{LYU2019, Zuo2021, WANG2023IC}, branch-and-bound optimization \cite{Huang2020DR}, and genetic algorithms \cite{Ma2021, Chen2021, Song2023}. Some of these methods obtain reasonably good routes and theoretical guarantees on the number of requests serviced or the stability of the system. However, these methods are not applicable to cases where new requests enter the system while the robots are already executing the planned routes, necessitating routes to adapt in real-time. Other works in the literature have tried to solve this limitation by considering real-time requests in their formulation, performing iterative optimizations that incorporate new requests at each time step. These works consider genetic algorithms \cite{WANG2020}, greedy insertion and route modification heuristics \cite{Gomes2014, NARAYAN2017, AlonsoMora2017, Sun2021}, relaxed optimization problems \cite{Huang2020CBO}, and proactive schemes that incorporate demand estimation to predict future requests \cite{BRUNI2014, Rongge2023, VANENGELEN2018, Wilbur2022}. Even though these methods are able to modify routes online based on real-time demand, the online algorithms under consideration do not longer provide provable stability claims. In this paper, we aim to fill this gap by proposing a framework that adapts to real-time demand online by considering potential future requests, but is still amenable to theoretical analysis and hence can provably guarantee stability for our learned policy.

\subsection{Fleet Sizing Algorithms}
Different works in the literature have tried to solve the problem of finding a fleet size that will achieve either a specific service rate or some notion of stability. Most works in the field have looked at the single-capacity-robot setting, where a robot must pick up and drop-off a request before being able to service a different request. For this setting, formal theoretical bounds on the number of robots required for stability of the routing policy have been derived \cite{Spieser2014, Garces2024ICRA}. These bounds, however, are not easily generalizable to the multi-capacity robot routing setting, where robots can pick-up multiple requests before dropping-off the first requests. Other works have tried to characterize the number of robots required for a specific request completion rate in the multi-capacity robot setting \cite{Zardini2022}. These approaches use simulation studies \cite{Barrios2014, Spieser2014, Bosch2016, Fagnant2018DynamicRA, Vazifeh2018AddressingTM}, path cover heuristics \cite{Qu2022HowMV, MAKHDOMI2024}, and iterative assignment formulations \cite{WINTER2018}. However, these approaches do not provide any analytical results or theoretical service rates or stability guarantees. Other works in the literature have tried to obtain theoretical service rate guarantees by considering an iterative joint optimization in the context of flow optimization \cite{Wallar2019OptimizingVD}, jointly solving the routing and the fleet sizing problems using historical data. 
These joint optimization approaches, however, are designed to produce fixed routes and fleet-sizes that are determined a-priori, and hence they are not applicable to our routing setting where routes need to be adapted in real-time, but the fleet size needs to be determined a-priori. 
In our work, we take inspiration from the algorithm proposed in \cite{Wallar2019OptimizingVD} to formulate an iterative fleet-sizing algorithm that assigns requests sequentially and grows the fleet size if a request can't be allocated to any robot. We show that such an approach does not procure the stability of our proposed greedy base policy. We then propose an improved algorithm that restarts the assignment process once the fleet size has been increased and we show that this algorithm provably finds a fleet size that is sufficiently large for our greedy base policy to be stable. 

\subsection{Demand Prediction Methods}
Different works in the literature have tried to predict demand for pickup-and-delivery services using historical data. Earlier works considered traditional machine learning schemes including XGBoost \cite{Shun2022}, Random Forests, Ridge Regression, and combination forecasting \cite{Liu2020CFM}. Recent advances in more advanced architectures inspired a push for demand prediction using LSTM-based models \cite{Huan2022}, and Generative Adversarial Networks (GANs) \cite{wang2023RH}. Other works have tried to use the structure of the environment to formulate the routing problem as a graph learning problem, introducing Graph Convolutional Networks (GCNs) \cite{Zhao2022GCN, Feng2022, STPGCN}. On the other hand, due to the time-dependence of demand patterns, especially in transportation settings, different authors have also considered casting the problem as a time-series and using time-series transformers as the predictive mechanism \cite{wu2023timesnet, nie2023a, wang2024timemixer, liu2024itransformer}. These approaches have seen recent success in prediction of traffic patterns and transportation requests for on-demand mobility. In addition to this, there have been empirical studies that hint at the possibility of transformers being better suited to forecast time-series with long term seasonal patterns compared to other architectures \cite{Zhao2021Predictability, haoyietal-informer-2021, haoyietal-informerEx-2023}. In this paper, we build on top off state-of-the-art time-series architectures \cite{wu2023timesnet, nie2023a, wang2024timemixer, liu2024itransformer} by conditioning on contextual data like the weather and date to predict demand for pickup-and-delivery tasks. We decide to use transformer-based architectures instead of Graph-based architectures due to the size of the environment and the fact that pickup-and-delivery requests tend to be distributed across different areas, which results in low-predictability for the spatiotemporal relationships between requests. In other words, the fact that requests tend to have spatially different pickup and drop-off locations over time results in sparse graph-based data that makes it harder for the GCNs to learn usable patterns. Time-series transformers, on the other hand, can aggregate spatial data, focusing exclusively on temporal patterns. This choice tends to decrease the sparsity of the data and make the learning process easier.

\section{Problem Formulation}
In this section, we present the formulation for a proactive multi-capacity robot routing problem, where routes and pickup times need to be optimized to service both scheduled requests as well as real-time requests. In the following subsections we define the environment, the request structure, the robot routes, the state and control spaces for the system, and the cost structure. Following a similar argument as in \cite{Bertsekas2022AlphaZero}, we formulate the robot routing problem as a stochastic dynamic program instead of a Markov Decision Process to more naturally connect the problem setup with the notation commonly used for rollout algorithms \cite{bertsekas2020rollout, Bertsekas2023RLCourse}.

\subsection{Environment}
We assume that our system operates on a finite, discrete time horizon, and we denote this time horizon as $T$. We assume that our system operates over a set of days that we denote as $\mathcal{D}$. For a given day $d \in \mathcal{D}$, we denote the start of the time horizon as $T^{\text{start}} \geq 0$, and the end as $T^{\text{end}}$, such that $T^{\text{end}}-T^{\text{start}} = T$. We refer to the time range $\vec{T} = [T^{\text{start}}, T^{\text{start}} + 1, \dots, T^{\text{end}}-1, T^{\text{end}}]$ as the system's operating times for a day $d$. We assume that start times and end times for different days are defined using a global time, and hence the start time $T^{\text{start}}$ for day $d+1$ will be larger than the end time $T^{\text{end}}$ for day $d$.

We consider an environment with a fixed physical structure that can be represented as a directed graph $G = (V, E)$, where $V = \{ 1, \dots, n\}$ corresponds to the set of physical locations of interest numbered $1$ through $n$, while $E \subseteq \{ (i,j) | i, j \in V\}$ corresponds to the set of directed segments connecting adjacent locations. 
For example, in a ground vehicle routing application, $G$ would correspond to the underlying street network, where the vertices $V$ correspond to street intersections, while the set of edges $E$ would correspond to the directed streets that connect adjacent intersections. 
We denote the set of neighboring locations to location $i$ as $\mathcal{N}_i = \{ j | j \in V, (i,j) \in E\}$. 
We define the time it takes to go from location $i$ to location $j \in \mathcal{N}_i$ as $\textit{time}(i,j)$. We denote the set of edges that composes the shortest path between nonadjacent locations $k$ and $l$ as $\mathcal{V}_{(k,l)} \subset E$, and the time associated with the path as $\textit{time}(k,l) = \sum_{(i,j) \in \mathcal{V}_{(k,l)}} \textit{time}(i,j)$.

We assume that we have a fixed fleet of $M$ autonomous multi-capacity robots represented by the set of identifiers $\mathcal{L} = \{ \ell_1, \ell_2, \dots, \ell_M\}$. 
All robots are assumed to be homogeneous with a known maximum capacity denoted as $C$. 
We assume that all autonomous robots have access to a central server and hence can obtain other robots' location, remaining capacity, and assigned requests, as well as all outstanding requests in the system.
We assume that robots are originally stored in depots or bases $\mathcal{B} \subset V$ that correspond to physical storing facilities in the map. We assume that these depots $\mathcal{B}$ are given a-priori. We denote the depot where robot $\ell_m$ is stored as $b^{\ell_m} \in \mathcal{B}$.
For a given day $d$, we assume that all robots are in circulation for the operating times $\vec{T}$ and that they remain inside the environment graph $G$. We also assume that all robots must return to their original depot before the end of the episode $T^{\text{end}}$.

\subsection{Request Structure}
\label{subsec:transportation_requests}
We define a request $r_q$ as a pickup and delivery task, and we represent it as a tuple, such that: 
\begin{equation*}
    r_q = \left< \rho_{q}, \delta_{q}, t_{q}^{\text{entry}}, t_{q}^{\text{pick}}, a_q, \widehat{t}^{\text{pick}}_{q} , \widehat{t}^{\text{drop}}_{q}, \psi^{\text{pick}}_{q}, \psi^{\text{drop}}_{q} \right>
\end{equation*}
where $\rho_{q} \in V$ is request $r_q$'s pickup location; $\delta_{q} \in V$ corresponds to request $r_q$'s drop-off location; $t_{q}^{\text{entry}}$ is the time at which the request entered the system; $t_{q}^{\text{pick}}$ corresponds to request $r_q$'s desired pickup time; $a_q \in \mathcal{L} \cup \{ -1\}$ corresponds to the label of the robot assigned to request $r_q$ or $-1$ if request $r_q$ hasn't been assigned to any robot; $\widehat{t}^{\text{pick}}_{q} \in \mathds{R}^{+} \cup \{-1\}$ corresponds to the planned pickup time as determined by the route of the assigned robot, where $\widehat{t}^{\text{pick}}_{q} \geq t_{q}^{\text{pick}}$ if the request $r_q$ has been assigned to a robot or $\widehat{t}^{\text{pick}}_{q} = -1$ otherwise; $\widehat{t}^{\text{drop}}_{q} \in \mathds{R}^{+} \cup \{-1\}$ corresponds to the planned drop off time as determined by the route of the assigned robot, where $\widehat{t}^{\text{drop}}_{q} > \widehat{t}^{\text{pick}}_{q}$ if the request $r_q$ has been assigned to a robot or $\widehat{t}^{\text{drop}}_{q} = -1$ otherwise; $\psi^{\text{pick}}_q \in \{ 0, 1\}$ is an indicator, such that $\psi^{\text{pick}}_q = 1$ if request $r_q$ has been picked up by a robot, and $\psi^{\text{pick}}_q = 0$ otherwise; and $\psi^{\text{drop}}_q \in \{ 0, 1\}$ is an indicator, such that $\psi^{\text{drop}}_q  = 1$ if request $r_q$ has been dropped off, and $\psi^{\text{drop}}_q  = 0$ otherwise. We assume that all requests must enter the system at a time prior to the requested pickup time, or more formally $t_{q}^{\text{entry}} < t_{q}^{\text{pick}}, \forall r_q$. We define the latest pickup time that can be specified for a request for a given day $d$ as $T^{\text{last}}$, such that $T^{\text{last}} < T^{\text{end}}$. The time $T^{\text{last}}$ is chosen in such a way that a robot at a depot has enough time to pickup and drop-off a request that enters the system at time $T^{\text{last}}$, while still being able to make it back to the depot on or before $T^{\text{end}}$.

We group requests based on whether they are scheduled ahead of the current operating times for day $d$, or they enter the system in real-time. We denote the set of requests that are scheduled ahead of the current operating time range $\vec{T}$ for day $d$ as  $\mathcal{R}^{\text{sched}} = \{ r_q | t_{q}^{\text{entry}} < T^{\text{start}}, T^{\text{start}} \leq t_{q}^{\text{pick}} \leq T^{\text{last}}\}$. 
This set of requests is assumed to be known at the start of the current operating times $\vec{T}$. 
We denote the set of real-time requests that enter the system at time $t$ for a given day $d$ as $\mathcal{R}^{\text{realtime}}_{t} = \{ r_q | t = t_{q}^{\text{entry}}, T^{\text{start}} \leq t_{q}^{\text{pick}} \leq T^{\text{last}}\}$.
We denote the set of all requests that have entered the system before time $t$ as $\mathcal{R}_{t} =  \left( \bigcup_{t' \in [T^{\text{start}}, \dots, t-1]} \mathcal{R}^{\text{realtime}}_{t'} \right) \cup \mathcal{R}^{\text{sched}}$.

We model the number of real-time requests that enter the system as a random variable $\eta$ with its realization at time $t$ being denoted as $\eta(t)$, and its unknown underlying distribution being denoted as $p_{\eta}$. From this, we have that $|\mathcal{R}^{\text{realtime}}_{t}| = \eta(t)$. We model the pickup $\rho_q$ and drop-off $\delta_q$ for requests $r_q \in \mathcal{R}^{\text{realtime}}_{t}$ as random variables $\rho$ and $\delta$, respectively. These random variables have unknown probability distributions $p_{\rho}$ and $p_{\delta}$, respectively. 
We assume that we have access to a set of historical dates $\mathcal{D}_{H} \subset \mathcal{D}$ that contains a subset of the operating dates for the system. We define the set of historical requests for the set of dates $\mathcal{D}_{H}$ as $\mathcal{H} = \{ \mathcal{R}_{T^{\text{end}}} |  T^{\text{end}} \text{is the end time for day }d, d \in \mathcal{D}_{H}\}$. We assume that the set $\mathcal{H}$ of historical requests can be  used to estimate $p_{\eta}$, $p_{\rho}$ and $p_{\delta}$ to obtain approximate probability distributions $\tilde{p}_{\eta}$, $\tilde{p}_{\rho}$ and $\tilde{p}_{\delta}$, respectively.
We define the true underlying request distribution as a tuple $\mathcal{P} = \left< p_{\eta}, p_{\rho}, p_{\delta} \right>$, and the estimated request distribution as $\tilde{\mathcal{P}} = \left< \tilde{p}_{\eta}, \tilde{p}_{\rho}, \tilde{p}_{\delta} \right>$.

\subsection{Robot Routes}
\label{subsec:robot_routes}
To represent robot routes for a given day $d$, we define a directed, weighted, time augmented-graph $G^{\text{time}} = (V^{\text{time}}, E^{\text{time}})$. We define $V^{\text{time}} = \{ \left< i, t\right> | \forall i \in V, T^{\text{start}} \leq t \leq  T^{\text{end}}\}$ to be the vertex set composed of a tuple of physical locations in the graph $G$ and time steps during the operating times $\vec{T}$ for day $d$. 
We define the edge set $E^{\text{time}} = E^{\text{move}} \bigcup E^{\text{wait}}$ to specify time transitions between locations in $G$, where:
\begin{align}
    E^{\text{move}} & = \{ ( \left< i, t \right>, \left< j, t' \right>) \big | (i,j) \in E, t \in \vec{T}, t'-t = \textit{time}{(i,j)}\} \nonumber \\
    E^{\text{wait}} & = \{ ( \left< i, t \right>, \left< i, t' \right>) \big | \forall i \in V, t \in \vec{T}, t'-t = 1 \} \nonumber
\end{align}
Intuitively, $E^{\text{move}}$ corresponds to the edges associated with movements between vertices in $G$, while $E^{\text{wait}}$ corresponds to a waiting action of one time-step at a specific vertex in $G$.

We define the set of requests that have been assigned to robot $\ell_{m}$ at time $t$ as $\mathcal{A}_{t}^{\ell_{m}} = \{ r_q | r_q \in \mathcal{R}_{t}, a_q = \ell_{m}\}$, and the vector of requests assignments for the entire fleet at time $t$ as $\vec{A}_{t} = [ \mathcal{A}_{t}^{\ell_{1}}, \dots, \mathcal{A}_{t}^{\ell_{M}} ]$. We define the set of valid request assignments for the entire fleet as $\mathcal{F}_{t}$, such that $\vec{A}_{t} \in \mathcal{F}_{t}$. The set of valid assignments $\mathcal{F}_{t}$ contains all possible assignments of requests to robots, where each request gets assigned to a single robot (i.e. $\mathcal{A}_{t}^{\ell_{k}} \cap \mathcal{A}_{t}^{\ell_{m}} = \emptyset$ for $\ell_{k} \neq \ell_{m}$). Also, once a request has been picked up by a robot, it cannot be reassigned to a different robot, and hence the set $\mathcal{F}_{t}$ must contain those assignments that were part of $\mathcal{F}_{t-1}$ and correspond to requests $r_q$ where $\psi_{q}^{\text{pick}} = 1$. 

We define the location of robot $\ell_{m}$ at time $t$ as $\nu_{t}^{\ell_{m}} \in V$, and the number of requests inside the robot at time $t$ as $c_{t}^{\ell_{m}}$. 
For each robot $\ell_{m} \in \mathcal{L}$, we define the space of valid routes at time $t$ as the set $\mathcal{Y}_{t}^{\ell_{m}}$. We define a valid route $y_{t}^{\ell_{m}} \in \mathcal{Y}_{t}^{\ell_{m}}$ as follows:
\begin{definition}
    A valid route $y_{t}^{\ell_{m}} \in \mathcal{Y}_{t}^{\ell_{m}}$ for a day $d$ is a directed path on $G^{\text{time}}$ that must satisfy the following four conditions:
    \begin{enumerate}
        \item $y_{t}^{\ell_{m}}$ starts on node $\left< b^{\ell_{m}}, T^{\text{start}}\right>$, ends on node $\left< b^{\ell_{m}}, T^{\text{end}}\right>$, and contains node $\left<\nu_{t}^{\ell_{m}}, t\right>$.
        \item $y_{t}^{\ell_{m}}$ must include the portion of $y_{t-1}^{\ell_{m}}$ that has already been traversed by time $t-1$.
        \item We denote the sequence of intermediate nodes in $y_{t}^{\ell_{m}}$ where requests are being picked up or dropped off as $\vec{S}_{t}^{\ell_{m}}$, and we refer to it as the sequence of stops associated with $y_{t}^{\ell_{m}}$. This sequence of stops $\vec{S}_{t}^{\ell_{m}}$  must include pickup nodes $\left< \rho_q, \widehat{t}^{\text{pick}}_{q}\right>$ with $\widehat{t}^{\text{pick}}_{q} \geq t_{q}^{\text{pick}}$ and drop-off nodes $\left< \delta_q, \widehat{t}^{\text{drop}}_{q}\right>$ with $\widehat{t}^{\text{drop}}_{q} > \widehat{t}^{\text{pick}}_{q}$ for all requests $r_q \in \mathcal{A}_{t}^{\ell_{m}}$.
        \item The number of requests inside the robot must be below the maximum capacity of the robot at all times (i.e. $c_{t}^{\ell_{m}} \leq C$ for $t \in \vec{T}$).
    \end{enumerate}
\end{definition}

We define the space of all possible valid robot routes for the fleet at time $t$ on day $d$ as $\mathcal{Y}_{t} = \mathcal{Y}_{t}^{\ell_{1}} \times \dots \times \mathcal{Y}_{t}^{\ell_{M}}$.

\subsection{State Representation and Control Space}

We represent the state of the system at time $t$ on day $d$ as a tuple $x_{t} = \left< \vec{\nu}_{t}, \vec{c}_{t}, \vec{y}_{t}, \vec{A}_{t}, \mathcal{R}_{t} \right>$, where: $\vec{\nu}_{t} = [\nu_{t}^{\ell_{1}}, \dots, \nu_{t}^{\ell_{M}}]$ corresponds to the vector of locations for all robots $\ell_{m} \in \mathcal{L}$ at time $t$; the vector $\vec{c}_t = [c_{t}^{\ell_{1}}, \dots, c_{t}^{\ell_{M}}]$ contains the number of requests inside each robot $\ell_{m} \in \mathcal{L}$ at time $t$; the vector $\vec{y}_{t} \in \mathcal{Y}_{t}$  contains the selected valid routes for all robots $\ell_{m} \in \mathcal{L}$ before the incoming requests $\mathcal{R}_{t}^{\text{realtime}}$ enter the system; the vector $\vec{A}_{t} \in \mathcal{F}_{t}$ contains the selected valid request assignments for the robot fleet $\mathcal{L}$ before the incoming requests $\mathcal{R}_{t}^{\text{realtime}}$ enter the system; and the set $\mathcal{R}_{t}$ corresponds to the requests that have entered the system before time $t$ on day $d$ (see Sec.~\ref{subsec:transportation_requests}). 

We define the set of possible controls for the robot fleet at time $t$ on day $d$ as $\mathcal{U}(x_{t}) = \{ [\vec{A}_{t+1}, \vec{y}_{t+1}] \big | \vec{A}_{t+1} \in \mathcal{F}_{t+1}, \vec{y}_{t+1} \in \mathcal{Y}_{t+1} \}$. Hence, at each time step $t \in \vec{T}$, the system must choose a specific control $u_{t} \in \mathcal{U}(x_{t})$ by specifying a valid request assignment for the fleet $\vec{A}_{t+1}$ and a valid set of routes $\vec{y}_{t+1}$ that consider incoming requests $\mathcal{R}^{\text{realtime}}_{t}$.

\subsection{Constraints and Cost Structure}
\label{subsec:contraints_cost}
Given that a request $r_q$ gets assigned to an arbitrary robot $\ell_{m}$ under control $u_{t}$ and state $x_{t}$, we define the wait time at the pick up station for request $r_q$ as $w^{\text{pick}}_{r_q}(x_{t}) = \widehat{t}^{\text{pick}}_{q}(x_{t}) - t_{q}^{\text{pick}}$, where $\widehat{t}^{\text{pick}}_{q}(x_{t})$ corresponds to request $r_q$'s planned pickup time given by the routes in state $x_{t}$, and $t_{q}^{\text{pick}}$ is request $r_q$'s desired pickup time. Similarly, we define the wait time inside the robot before request $r_q$ gets dropped off as $w^{\text{drop}}_{r_q}(x_{t}) = \widehat{t}^{\text{drop}}_{q}(x_{t}) - (\widehat{t}^{\text{pick}}_{q}(x_{t})+\textit{time}(\rho_q, \delta_q))$, where $\widehat{t}^{\text{drop}}_{q}(x_{t})$ is request $r_q$'s planned drop off time given by the routes in state $x_{t}$, and $\textit{time}(\rho_q, \delta_q)$ is the direct trip time between the pickup $\rho_q$ and drop off $\delta_q$ locations. We model request impatience by introducing hard constraints on the wait times at the pickup station and wait times inside the robot before a request gets dropped off. 
We define $W^{\text{pick}}$ and $W^{\text{drop}}$ as the maximum amount of time a request is willing to wait at the pickup station and inside the robot, respectively, before the request is canceled. 
If a request is assigned to a robot using a valid route as defined in Sec.~\ref{subsec:robot_routes}, but its assignment results in a violation of the maximum wait times for itself or any of the previously assigned requests, then the route is considered invalid. 
If request $r_q$ can't be assigned to any robot because its assignment results in an invalid route for every robot under consideration, then the request is rejected and we set $a_q = -1$, $w^{\text{pick}}_{r_q}(x_{t}) = \infty$, and $w^{\text{drop}}_{r_q}(x_{t}) = \infty$. Rejected requests are added to the set $\mathcal{R}^{\text{rejected}}$. 

Using these wait times, we define $h(x_{t})$ as the immediate cost of state $x_{t}$:
\begin{equation}
    \label{eq:immediate_cost}
    h(x_{t}) = \sum_{r_q \in \mathcal{R}_{t}} w^{\text{pick}}_{r_q}(x_{t}) + w^{\text{drop}}_{r_q}(x_{t})
\end{equation}
 where the cost is expressed as the sum of the wait times given by the routes contained in state $x_{t}$ for all requests that have entered the system before time $t$. If no request has entered the system before time $t$ then we set $h(x_{t}) = 0$. For time steps $t' < T^{\text{start}}$ we set $h(x_{t'}) = 0$.

 In the following sections we will cover the definition of stability that we use in this paper, and the challenges associated with the proactive and multi-capacity aspects of the routing problem.

\section{Stability of a Robot Routing Policy}
\label{subsec:stability_definition}
We define a policy $\pi = \{ \mu_{T^{\text{start}}}, \dots, \mu_{T^{\text{end}}}\}$ for a given day $d$ as a set of functions that produces control $u_{t} = \mu_t(x_{t}) \in \mathcal{U}(x_{t})$ based on the current state $x_{t}$. We denote the state transition function as $f$, such that $x_{t+1} = f(x_{t}, u_{t}, \mathcal{R}^{\text{realtime}}_{t})$ and robots move along the routes contained in $u_{t}$. We define the stage cost for this transition as:
\begin{equation}
    \label{eq:stage_cost}
    g(x_{t}, u_{t}, \mathcal{R}^{\text{realtime}}_{t}) = h(x_{t+1}) - h(x_{t})
\end{equation}
and the cost at the beginning of the horizon as 
\begin{equation}
    \label{eq:initial_stage_cost}
    g(x_{T^{\text{start}}}, u_{T^{\text{start}}}, \mathcal{R}^{\text{realtime}}_{T^{\text{start}}}) = h(x_{T^{\text{start}}+1})
\end{equation}
Since at time $t$ we do not have access to requests that haven't entered the system yet, we consider a stochastic formulation for the cost of executing a policy $\pi$. Similarly to \cite{Bertsekas2023RLCourse}, we define the expected cost of executing a policy $\pi$ from state $x_{t}$ to the end of the horizon $T^{\text{end}}$ as 
\begin{align}
    J_{\pi} (x_{t}) =  \mathds{E} \left[  \sum_{t' = t}^{T^{\text{end}}} g(x_{t'}, \mu_{t'}(x_{t'}), \mathcal{R}^{\text{realtime}}_{t'}) \right] \nonumber
\end{align}
where the expectation is taken over the random variables associated with the number and locations of future requests contained in the sets $\mathcal{R}^{\text{realtime}}_{t}$,$ \dots$, $\mathcal{R}^{\text{realtime}}_{T^{\text{end}}}$. We use this cost definition and the formulation given in \cite{Bertsekas2022AlphaZero} to define stability as follows:
\begin{definition}
    \label{def:stability}
    A policy $\pi$ is stable if $J_{\pi}(x_{t}) < \infty$ for all states $x_{t}$ for $t \in \vec{T}$
\end{definition}
This means that a policy $\pi$ for a given day $d$ is stable if the cost is bounded away from infinity for all reachable states during the operating times $\vec{T}$.

\section{Challenges of the Proactive Multi-Capacity Robot Routing Problem}
In this section, we will cover why the proactive multi-capacity robot routing problem is hard, and what are the specific challenges that we will focus on this paper. In general, we are interested in finding a policy $\pi^{*}$ such that
\begin{align}
    J_{\pi^{*}}(x_{t}) = \min_{\pi \in \Pi} J_{\pi}(x_{t}) \nonumber
\end{align}
for $t \in \vec{T}$, where $\Pi$ is the space of all policies for a given day $d$. Solving this minimization optimally to find $\pi^{*}$ is intractable since even simplified versions of the problem are NP-hard \cite{Lenstra2006}. For this reason we consider policy improvement schemes, such as rollout-based methods \cite{bertsekas2019reinforcement} \cite{bertsekas2020rollout}, which allow us to obtain a lower cost policy by improving upon an easy to compute heuristic known as a base policy. 
We consider rollout methods for tackling the proactive dynamic multi-capacity robot routing problem due to their promising results in combinatorial applications \cite{bertsekas2020rollout, Bhattacharya2020MultiagentRollout} and single-occupancy robot routing problems \cite{Garces2023ICRA, Garces2024ICRA, Garces2024AAMAS}, and the recent development of a theoretical framework that states that a rollout algorithm with a reasonable, stable base policy will learn a stable near-optimal policy \cite{Bertsekas2022AlphaZero}.

We define $\bar{\pi} = \{ \bar{\mu}_{T^{\text{start}}}, \dots, \bar{\mu}_{T^{\text{end}}}\}$ as the base policy of choice for a given day $d$. In the rollout setup, we are then interested in finding an approximate policy $\tilde{\pi} = \{ \tilde{\mu}_{T^{\text{start}}}, \dots, \tilde{\mu}_{T^{\text{end}}}\}$, such that the minimizing action for state $x_{t}$ is given by:

\begin{align}
    \Tilde{\mu}_{t}(x_{t}) \in \argmin_{u_{t} \in \mathcal{U}(x_{t})} & \mathds{E}[g(x_{t}, u_{t}, \mathcal{R}^{\text{realtime}}_{t}) + \Tilde{J}_{\bar{\pi}}(x_{t+1})]
    \label{eq:rollout}
\end{align}

Where $x_{t+1} = f(x_{t}, u_{t}, \mathcal{R}^{\text{realtime}}_{t})$ corresponds to the next state given by the transition function, and $\Tilde{J}_{\bar{\pi}}(x_{t+1}) = \sum_{t'=t+1}^{(t+1)+K} g(x_{t'}, \bar{\mu}_{t'}(x_{t'}), \mathcal{R}^{\text{realtime}}_{t'})$ is a cost approximation obtained by applying the base policy $\bar{\pi}$ for $K$ time steps. We estimate the expectation in Eq.~(\ref{eq:rollout}) using Monte-Carlo simulations. This formulation allows us to get a reasonable approximation, but it is still not amenable to the multi-capacity routing setting, where multiple requests can be assigned to the same robot even before the originally assigned requests are dropped off. To generalize this formulation to the multi-capacity routing setting, while allowing the rollout algorithm to consider potential future requests proactively, we must first solve the following four challenges:\\
\textbf{Challenge 1:} We must adapt the multi-agent rollout algorithm to handle the combinatorial dependence on the number of agents and the number of requests associated with the multi-capacity setting. \\
\textbf{Challenge 2:} We must construct a base policy for the multi-capacity setting that is computationally efficient and provably stable so the rollout algorithm can learn a near-optimal stable policy as suggested by previously developed theory \cite{Bertsekas2022AlphaZero}. \\
\textbf{Challenge 3:} We must design a mechanism for accurately estimating the number and locations of future requests and the times at which such requests might enter the system, in order to be able to approximate the expectation in Eq.~\ref{eq:rollout} accurately. \\
\textbf{Challenge 4:} We must develop an algorithm to selectively explore a subset of the control space in order to make the minimization in Eq.~\ref{eq:rollout} tractable while still generating assignments and routes that reflect real-time demand and improve upon the base policy.

In this paper, we propose a system that addresses these four challenges.

\section{Our Approach}
To solve \textbf{Challenge 1}, we propose a new variation of multi-agent rollout that performs assignments one request at a time. We call this new rollout algorithm one-request-at-a-time rollout. Compared to the standard one-agent-at-a-time rollout algorithm \cite{bertsekas2020rollout, Bertsekas2023RLCourse}, our algorithm does not iterate over each robot to find the subset of requests that will result in the lowest cost for the immediate assignment, but instead it assigns one request at a time to robots in the fleet. By doing this, our algorithm can iteratively build routes by considering single-request assignments instead of having to search a combinatorial growing action space associated with the generation of request subsets. In this sense, similarly to the standard one-agent-at-a-time rollout algorithm, our proposed algorithm exchanges complexity in the control space for complexity in the state space but it does so by considering individual request's assignments instead of individual agent's actions. This distinction is important since the fleet size is fixed and hence a single request must consider a constant number of robots in its assignment, while a robot that must choose a set of requests to service must consider a subset of the currently available requests which results in a combinatorial search.

To solve \textbf{Challenge 2}, we design a greedy base policy that is computationally efficient, but still allows for reasonable allocations of requests to robots. In the theoretical results section, we show that our proposed base policy is stable as long as we have a sufficiently large fleet size. We also provide an algorithm to find this sufficiently large fleet size using historical data, and show that our fleet sizing algorithm asymptotically guarantees stability of our proposed base policy for any request distribution as the number of historical days considered in the fleet sizing algorithm goes to infinity (i.e. we can perfectly estimate the underlying request distribution).

To solve \textbf{Challenge 3}, we design a generative model that uses time-series transformers to predict the number of requests at every time interval as a time-series. The pickup and drop-off locations for each request are determined using samples from empirically derived conditional distributions. 

Finally, to solve \textbf{Challenge 4}, we propose a promising-route generation mechanism that uses unsupervised clustering of currently available requests to determine potentially beneficial multi-request swaps across robots. The method then uses these swaps to generate potential routes for the robot fleet. 
In the following subsections, we will cover the solutions for all four challenges in more detail.

\subsection{One-request-at-a-time Rollout}
Our proposed one-request-at-a-time rollout algorithm modifies the standard one-agent-at-a-time rollout algorithm \cite{bertsekas2020rollout} \cite{Bhattacharya2020MultiagentRollout} \cite{Garces2023ICRA} to consider incoming requests as the agents. We consider requests that enter the system at a given time step $t$ one at a time sorting them based on their desired pickup time. For the first time step $t=T^{\text{start}}$ for day $d$, we consider both the set of scheduled requests $\mathcal{R}^{\text{sched}}$ and the set of incoming requests $\mathcal{R}^{\text{realtime}}_{T^{\text{start}}}$, considering the requests in $\mathcal{R}^{\text{sched}}$ first. 
As controls are chosen for requests, next requests in the order consider controls of previous requests as fixed and use controls dictated by the base policy as the controls for future requests. More specifically, let's consider a request $r_q$, which enters the system at time $t$ after $z$ other requests have been considered for assignment at time step $t$. From this, we have that its one-request-at-a-time rollout control is:
\begin{align}
    \tilde{u}_{t}^{q} \in \argmin_{u_{t}^{q} \in \mathcal{U}^{q} (x_{t})} & \mathds{E}[g(x_{t}, \bar{u}, \mathcal{R}^{\text{realtime}}_{t}) + \Tilde{J}_{\bar{\pi}}(x_{t+1})]
    \label{eq:orat_formulation}
\end{align}
where the set $\mathcal{U}^{q} (x_{t})$ is the space of controls available to request $r_q$ at state $x_{t}$ given that the previous $z$ requests already have selected their controls;
the vector
\begin{align*}
    \bar{u} = (\tilde{u}_{t}^{q-z}, \dots, u_{t}^{q}, \dots \bar{\mu}_{t}^{|\mathcal{R}^{\text{realtime}}_{t}|}(x_{t}))
\end{align*} is the sequence of controls for all requests in $\mathcal{R}^{\text{realtime}}_{t}$, where requests before $r_q$ have their respective minimizing controls, while requests after $r_q$ have their controls dictated by the base policy $\bar{\pi}$; and
\begin{align*}
    \Tilde{J}_{\bar{\pi}}(x_{t+1}) = \sum_{t'=t+1}^{(t+1)+K} g(x_{t'}, \bar{\mu}_{t'}(x_{t'}), \mathcal{R}^{\text{realtime}}_{t'})
\end{align*}
is a cost approximation obtained by applying the base policy $\bar{\pi}$ for $K$ time steps.

Now that we have the one-request-at-a-time formulation, we can move on to cover the details of the base policy selected for this problem, the generative model used to generate future samples to accurately approximate the expectation in Eq.\ref{eq:orat_formulation}, and the promising-route generator that is used to create the set of feasible controls $\mathcal{U}^{q} (x_{t'})$ for request $r_q$.

\subsection{Base policy}
\label{subsec:greedy_base_policy}
In order to leverage the theoretical guarantees given by \cite{Bertsekas2022AlphaZero}, which state that a rollout policy with a stable base policy will learn a stable near-optimal policy, we are interested in designing a stable base policy that is computationally efficient and produces reasonable routes. To achieve this, we propose a greedy base policy that we denote as $\pi^{\text{greedy}}$, and we set it as the base policy for the one-request-a-time rollout algorithm such that $\bar{\pi} = \pi^{\text{greedy}}$. The greedy base policy $\pi^{\text{greedy}}$ inserts every request following the algorithm described in algo.~\ref{algo:greedy_base_policy}. This algorithm uses the methods \textit{GetRobotRoute}$(x_{t}, \ell_{m})$, \textit{InsertionProcedure}($y_{t}^{\ell_{m}}, \ell_{m}, r_q$), and \textit{UpdateRobotRoute}($x_{t}, \tilde{y}_{t}^{\ell_{m}}, \ell_{m}$). 

\begin{algorithm}
    \caption{Greedy base policy}
    \begin{algorithmic}[1]
    \label{algo:greedy_base_policy}
        \REQUIRE current state $x_{t}$, incoming request $r_q$.
        \ENSURE the greedy control $u^{\text{greedy}}_{t}$.
        \STATE $U = []$
        \FOR{$\ell_{m} \in \mathcal{L}$}
            \STATE $y_{t}^{\ell_{m}} = \textit{GetRobotRoute}(x_{t}, \ell_{m})$
            \STATE $\tilde{y}_{t}^{\ell_{m}} = \textit{InsertionProcedure}(y_{t}^{\ell_{m}}, \ell_{m}, r_q)$
            \IF{$\tilde{y}_{t}^{\ell_{m}}$ is not None}
                \STATE $\vec{A}_{t}, \vec{y}_{t} = \textit{UpdateRobotRoute}(x_{t}, \tilde{y}_{t}^{\ell_{m}}, \ell_{m})$
                \STATE $u_{t} = [\vec{A}_{t}, \vec{y}_{t}]$
                \STATE $U.\text{append}(u_{t})$
            \ENDIF
        \ENDFOR
        \IF{U is empty}
        \RETURN None
        \ELSE
        \STATE $u^{\text{greedy}}_{t} = \argmin_{u' \in U}g(x_{t}, u', \{r_q\})$
        \RETURN $u^{\text{greedy}}_{t}$
        \ENDIF
    \end{algorithmic}
\end{algorithm}

The method \textit{GetRobotRoute}$(x_{t}, \ell_{m})$ extracts the robot route $y_{t}^{\ell_{m}}$ for robot $\ell_{m}$ from the fleet routes in state $x_{t}$. The method \textit{InsertionProcedure}($y_{t}^{\ell_{m}}$, $\ell_{m}$, $r_q$) inserts request $r_q$'s pickup and dropoff locations in the route $y_{t}^{\ell_{m}}$, using the procedure described in algo.~\ref{algo:insertion_procedure}. 
Finally, The method \textit{UpdateRobotRoute}($x_{t}$, $\tilde{y}_{t}^{\ell_{m}}$, $\ell_{m}$) updates the routes inside state $x_t$ using the new route $\tilde{y}_{t}^{\ell_{m}}$ produced by the insertion procedure to generate a new assignment vector $\vec{A}_{t}$ and new routes $\vec{y}_{t}$ for the entire fleet. If the greedy policy returns None, the old routes in state $x_{t}$ are kept and the request $r_q$ gets added to the set $\mathcal{R}^{\text{rejected}}$.

\begin{algorithm}
    \caption{InsertionProcedure}
    \begin{algorithmic}[1]
    \label{algo:insertion_procedure}
        \REQUIRE current route $y_{t}^{\ell_{m}}$, agent identifier $\ell_{m}$, incoming request $r_q$.
        \ENSURE new planned route $\tilde{y}_{t}^{\ell_{m}}$ for agent $\ell_{m}$.
        \STATE $Y = []$
        \STATE $\rho_q, \delta_q = \textit{GetPickDropLocations}(r_q)$
        \STATE $\vec{S}_{t}^{\ell_{m}} = \textit{GetStopSeq}(y_{t}^{\ell_{m}})$
        \FOR{\textit{PickIndex} in $[0, \dots \textit{length}(\vec{S}_{t}^{\ell_{m}})-2]$}
            \STATE $\vec{S}_{t}^{\text{pick}} = \textit{InsertStop}(\rho_q, \textit{PickIndex}, \vec{S}_{t}^{\ell_{m}})$
            \FOR{\textit{DropIndex} in $[\textit{PickIndex}, \dots, \textit{length}(\vec{S}_{t}^{\ell_{m}})-2]$}
                \STATE $\vec{S}_{t}^{\text{drop}} = \textit{InsertStop}(\delta_q, \textit{DropIndex}, \vec{S}_{t}^{\text{pick}})$
                \STATE $y_{t} = \textit{GenerateRoute}(\vec{S}_{t}^{\text{drop}})$
                \IF{$\textit{ValidRoute}(y_{t})$}
                    \STATE $Y.\text{append}(y_{t})$
                \ENDIF
            \ENDFOR
        \ENDFOR
        \STATE $\tilde{y}_{t}^{\ell_{m}} = \textit{GetMinCostRoute}(Y)$
        \RETURN $\tilde{y}_{t}^{\ell_{m}}$
    \end{algorithmic}
\end{algorithm}

The method \textit{InsertionProcedure}($y_{t}^{\ell_{m}}$, $\ell_{m}$, $r_q$) (see algo.~\ref{algo:insertion_procedure}) uses $6$ main functions. The function $\textit{GetPickDropLocations}(r_q)$ obtains the pickup location $\rho_q$ and drop-off location $\delta_q$ associated with request $r_q$.
The function $\textit{GetStopSeq}(y_{t}^{\ell_{m}})$ retrieves the stop sequence $\vec{S}_{t}^{\ell_{m}}$ (as defined in Sec.~\ref{subsec:robot_routes}) associated with the current route $y_{t}^{\ell_{m}}$ for agent $\ell_{m}$. 
The function \textit{InsertStop}($\rho_q$, $\textit{PickIndex}$, $\vec{S}_{t}^{\ell_{m}}$) inserts $\rho_q$ as a stop in the stop sequence $\vec{S}_{t}^{\ell_{m}})$, placing it on index $\textit{PickIndex}+1$ and shifting all the stops in the sequence that previously occupied indices higher than or equal to $\textit{PickIndex}+1$. 
Similarly, the function \textit{InsertStop}($\delta_q$, \textit{DropIndex}, $\vec{S}_{t}^{\text{pick}}$) has the same functionality as the previous \textit{InsertStop} function, but considers the drop-off location of the request $\delta_q$, and the modified stop sequence $\vec{S}_{t}^{\text{pick}}$ instead. 
The function \textit{GenerateRoute}($\vec{S}_{t}^{\text{drop}}$) creates a route by constructing a directed path over $G^{\text{time}}$ using the shortest path between adjacent stops in the sequence of stops $\vec{S}_{t}^{\text{drop}}$. 
Robots that reach a pickup location before the desired pickup time of a request use edges in $E^{\text{wait}}$ to wait at that location until the request's desired pickup time. 
We use the function $\textit{ValidRoute}(y_{t})$ to check that the generated route $y_t$ is valid, meaning that it does not violate the structural conditions of a valid route described in Sec.~\ref{subsec:robot_routes} or the time constraints described in Sec.~\ref{subsec:contraints_cost}. 
Once we have iterated over all possible insertion points in the sequence of stops, and have generated a list of valid routes $Y$, we select the route with the minimum wait times by using the function  $\textit{GetMinCostRoute}(Y)$. 
We consider the wait times of requests as described in Sec.~\ref{subsec:contraints_cost}. Note that all original requests that were part of the route must still be part of the route after request $r_q$ is inserted, and hence the cost associated with the new route is positive. 
If $Y$ is empty, then $\tilde{y}_{t}^{\ell_{m}}$ is set to None. We show an example of this insertion procedure for a route with a single stop before the depot in Fig.~\ref{fig:insertion_example}.

\begin{figure}
    \centering
    \includegraphics[width=0.99\linewidth]{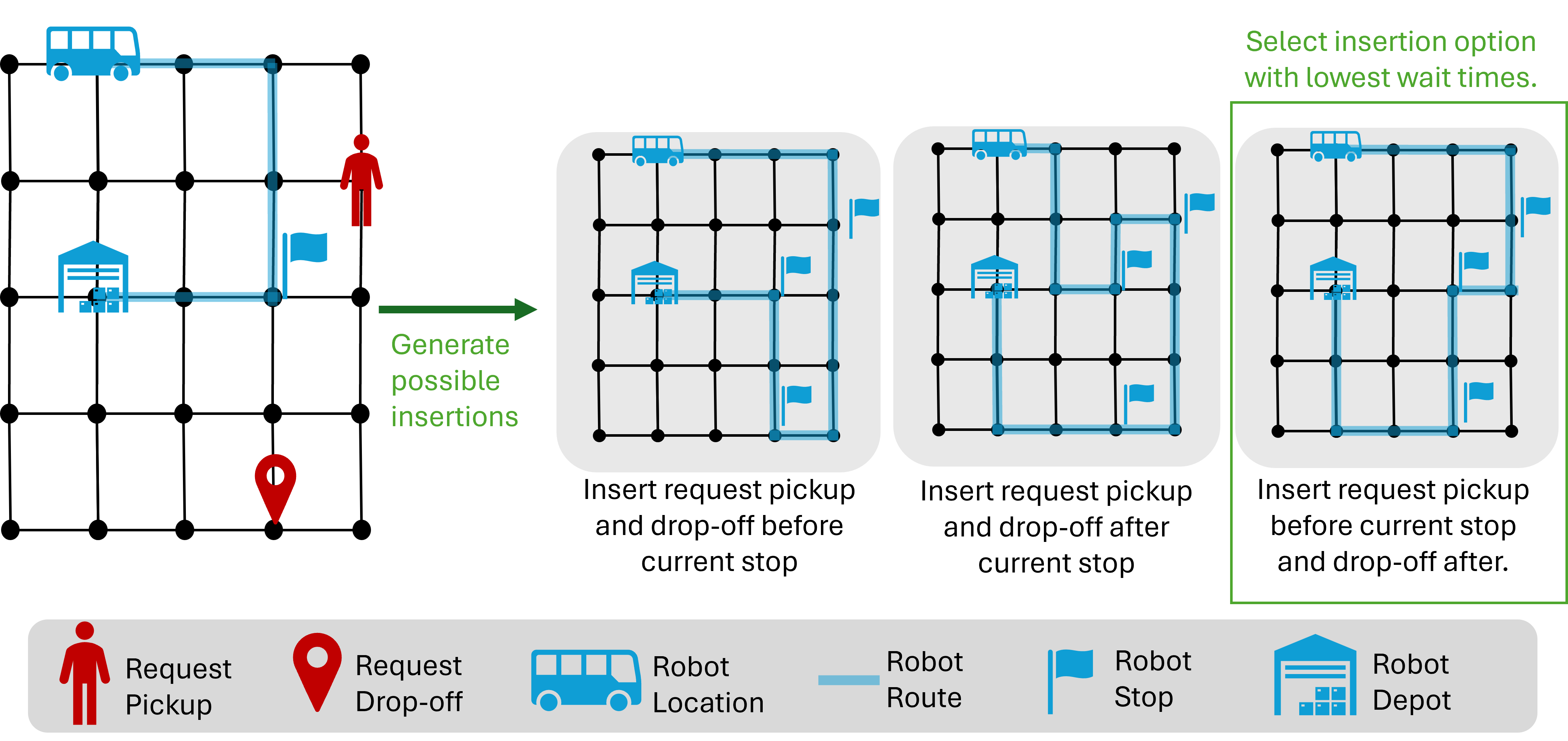}
    \caption{Example of executing the insertion procedure detailed in algo.~\ref{algo:insertion_procedure}, when there is only one stop in the bus route besides the depot.}
    \label{fig:insertion_example}
\end{figure}

In the theoretical results section (see Sec.~\ref{sec:theoretical_results}), we will show that our proposed greedy base policy $\pi^{\text{greedy}}$ is stable as long as we have a sufficiently large fleet size, and we will provide a fleet sizing algorithm that provably finds this sufficiently large fleet size.

\subsection{Generative Model for Predicting Future Requests}
\label{subsec:generative_model}
In order to better estimate future costs associated with the assignment of future requests, we propose a generative model for predicting potential future requests. Our proposed generative model is composed of three main components: 1) a time-series transformer used to predict the number of requests entering the system for the next hour using the requests observed during the current hour, 2) A conditional distribution for the location of the pickups estimated using historical data, and 3) A conditional distribution for the location of the drop-offs also estimated using historical data. We will cover the details of these three components in the following subsections.

\subsubsection{Time-series Transformer for predicting the time and number of requests}
For the task of predicting the number of real-time requests in the future, we consider the current hour of operation. We assume we can partition the current hour $h$ of operation for day $d$ into $N$ non-overlapping intervals of equal size $[I^{1}_{h}, \dots, I^{N}_{h}]$. For an arbitrary interval $I^{k}_{h}$, we denote the number of requests inside the interval as $R^{k}_{h} = \sum_{t \in I^{k}_{h}} |\mathcal{R}^{\text{realtime}}_{t}|$. We denote the set of context variables associated with each interval $I^{k}_{h}$ as $O^{k}_{h}$. These context variables include temporal identifiers like the hour, day of the week, month, and interval index, in addition to weather information including temperature and precipitation for that specific hour. For each time interval, we generate a feature vector $Z^{k}_{h} = [R^{k}_{h}, O^{k}_{h}]$. We denote the time series of $N$ feature vectors for the current hour $h$ as $\vec{Z}_{h} = [Z^{1}_{h}, \dots, Z^{N}_{h}]$ and the time-series of context variables for the next hour as $\vec{O}_{h+1} = [O^{1}_{h+1}, \dots, O^{N}_{h+1}]$. We modify the time-series transformer architectures \cite{wu2023timesnet} \cite{nie2023a} \cite{wang2024timemixer} \cite{liu2024itransformer} \cite{wang2024tssurvey} to handle the context variables, allowing us to feed $\vec{Z}_{h}$ and $\vec{O}_{h+1}$ directly into the transformer to obtain a predicted sequence of requests $R^{\text{pred}}_{h+1} = [\tilde{R}^{1}_{h+1}, \dots, \tilde{R}^{N}_{h+1}]$, where $\tilde{R}^{k}_{h+1}$ is the predicted number of requests for interval $k$ in the following hour $h+1$. This sequence of requests determines the number of samples that will be obtained from the two conditional distributions covered in the following subsection.

\subsubsection{Conditional distributions for the pickup and drop-off locations}
\label{subsubsec:conditional_dist}
Given historical data we compute the empirical distribution for the location of the pickups by conditioning on a specific month $s_m$, a specific day of the week $s_w$, and a specific hour of the day $s_h$.
We define the number of historical requests with pickups at intersection $i \in V$ for a specific month $s_m$, a specific day of the week $s_w$, and a specific hour of the day $s_h$, as $n^{\text{pick}(i)}_{(s_m, s_w, s_h)}$. Using this, we define the probability of a request $r_q$'s pickup occurring at location $i$ given that the request enter the system during a specific month $s_m$, a specific day of the week $s_w$, and a specific hour of the day $s_h$ as:
\begin{align*}
    Pr(\rho_q = i | s_m, s_w, s_h) = \frac{n^{\text{pick}(i)}_{(s_m, s_w, s_h)}}{\sum_{j \in V} n^{\text{pick}(j)}_{(s_m, s_w, s_h)}}
\end{align*}

To obtain the conditional distribution for the drop-off locations, we condition on a specific month $s_m$, a specific day of the week $s_w$, and the specific pickup location $\rho$ of historical requests. We define the number of historical requests with drop-offs at intersection $i \in V$ for a specific month $s_m$, a specific day of the week $s_w$, and with a specific pickup $\rho$ as $n^{\text{drop}(i)}_{(s_m, s_w, \rho)}$. We hence define the probability of a request $r_q$'s drop-off occurring at location $i$ given that its pickup happened at location $\rho$, and the request entered the system during a specific month $s_m$, and a specific day of the week $s_w$ as:
\begin{align*}
    Pr(\delta_q = i | s_m, s_w, \rho_q=\rho) = \frac{n^{\text{drop}(i)}_{(s_m, s_w, \rho)}}{\sum_{j \in V} n^{\text{drop}(j)}_{(s_m, s_w, \rho)}}
\end{align*}

\subsubsection{Sampling potential future requests}
We are interested in generating sequences of potential future requests to be used in the approximation of the expectation in Eq.~\ref{eq:orat_formulation}. We consider a sequence of future requests for the next hour of operation as a single future scenario. For this reason, we are interested in generating multiple future scenarios by sampling multiple sequences of requests.
We define the total number of future scenarios to be considered as $N^{\text{scenarios}}$. 
To generate the samples of potential future requests that compose a single scenario, we first consider the number of expected future requests for the next hour $R^{\text{pred}}_{h+1}$ which is obtained from the time-series transformer. We assume that if a request was predicted to come into the system during time interval $I^{k}_{h+1}$, it will enter the assignment process at the end of such time interval. For each predicted request, we use the conditional distributions described in Sec.~\ref{subsubsec:conditional_dist} to sample first the pickup location using the month, day of the week, and hour contained in the context variables $O^{k}_{h+1}$. We then use the sampled pickup location as a conditioning variable for the drop-off distribution. We sample the drop-off location for the predicted request using the sampled pickup location, and the month and day of the week contained in the context variables $O^{k}_{h+1}$. We repeat this process as many times as the number of requests predicted by the time-series transformers. All the sampled requests for a single future scenario are grouped in the set $\mathcal{R}^{\text{future}}_{l}$ where $l$ corresponds to the index of the future scenario. We denote the final vector of future scenarios as $R^{\text{future}} = [\mathcal{R}^{\text{future}}_{1}, \dots, \mathcal{R}^{\text{future}}_{N^{\text{scenarios}}}]$, and we use it in combination with the base policy to approximate the expectation of our proposed rollout framework contained in Eq.~\ref{eq:orat_formulation}.

\subsection{Promising-routes generator}
\label{subsec:promising_action_generator}
In order to generate the set of actions $\mathcal{U}^{q}(x_{t})$ to be considered in the one-request-at-a-time rollout framework, we propose a promising-routes generator described in algo.~\ref{algo:promising_actions}. This promising-routes generator uses the methods $\textit{InsertReq}(r_q, x_t, \ell_m)$, $\textit{ExtractReqs}(x_t)$, \textit{ClusterReqs}(\textit{AvailReqs}), \textit{IsClusterValid}(\textit{Cluster}), \textit{RemoveReqs} (\textit{Cluster}, $x_{t}$, $\ell_{m}$), 
\textit{InsertReqs}(\textit{Cluster}, $\vec{A'}_{t}$, $\vec{y'}_{t}$, $\ell_{k}$), and \textit{GetMinControls}($U$, $N^{\text{routes}}$). 

\begin{algorithm}[!ht]
    \caption{Promising-routes generator}
    \begin{algorithmic}[1]
    \label{algo:promising_actions}
        \REQUIRE current state $x_{t}$, current request $r_q$, maximum number of controls $N^{\text{routes}}$
        \ENSURE Set of promising actions $\mathcal{U}^{q}(x_{t})$
        \STATE U = []
        \FOR{$\ell_{m} \in \mathcal{L}$}
        \STATE $\vec{A}_{t}, \vec{y}_{t} = \textit{InsertReq}(r_q, x_t, \ell_m)$
        \IF{$\vec{y}_{t}$ is not None}
            \STATE $\bar{u}_{t} = [\vec{A}_{t}, \vec{y}_{t}]$
            \STATE $U.\text{append}(\bar{u}_t)$
            \STATE $\mathcal{R}_{t} = \textit{ExtractReqs}(x_t)$
            \STATE $\textit{AvailReqs} = \{ r_l | r_l \in \mathcal{R}_{t}, \psi^{\text{drop}}_l = 0, a_q = \ell_{m}\}$
            \STATE $\textit{ReqClusters} = \textit{ClusterReqs}(\textit{AvailReqs})$
            \FOR{\textit{Cluster} in \textit{ReqClusters}}
                \IF{\textit{IsClusterValid}(\textit{Cluster})}
                    \STATE $\vec{A'}_{t}, \vec{y'}_{t} = \textit{RemoveReqs}(\textit{Cluster}, x_{t}, \ell_{m})$
                    \FOR{$\ell_{k} \in \mathcal{L} \setminus \{ \ell_{m}\}$}
                    \STATE $\vec{A}_{t}, \vec{y}_{t} = \textit{InsertReqs}(\textit{Cluster}, \vec{A'}_{t}, \vec{y'}_{t}, \ell_k)$
                    \IF{$\vec{y}_{t}$ is not None}
                        \STATE $u_{t} = [\vec{A}_{t}, \vec{y}_{t}]$
                        \STATE $U.\text{append}(u_t)$
                    \ENDIF
                    \ENDFOR
                \ENDIF
            \ENDFOR
        \ENDIF
        \ENDFOR
        \STATE $\mathcal{U}^{q}(x_{t}) = \textit{GetMinControls}(U, N^{\text{routes}})$
        \RETURN $\mathcal{U}^{q}(x_{t})$
    \end{algorithmic}
\end{algorithm}

To generate a promising route, we start by considering an insertion without any request swaps. Algo.~\ref{algo:promising_actions} uses the method $\textit{InsertReq}(r_q, x_t, \ell_m)$ to insert a request in the route for robot $\ell_m$ using the greedy policy described in algo.~\ref{algo:greedy_base_policy}. Once we have appended the simple insertion to the list of possible routes, we consider requests assigned to robot $\ell_m$ that have not been dropped off yet.
For each request $r_l \in \textit{AvailReqs}$, we generate a feature vector $[\rho_l, \delta_l, t^{\text{pick}}_l]$. This feature vector is then used by the method \textit{ClusterReqs}(\textit{AvailReqs}) to cluster requests using the HDBSCAN clustering algorithm \cite{Campello2015} \cite{Malzer2020}. 
Each cluster of requests is checked using the \textit{IsClusterValid}(\textit{Cluster}) method, which verifies if all requests in the cluster are requests that have not been picked up yet. 
If all requests in a cluster have not been picked up, then the requests get unallocated from robot $\ell_{m}$ using $\textit{RemoveReqs}(\textit{Cluster}, x_{t}, \ell_{m})$, and then get allocated to a new robot $\ell_{k}$ using the method \textit{InsertReqs}(\textit{Cluster}, $\vec{A'}_{t}, \vec{y'}_{t}, \ell_k$). 
The $\textit{RemoveReqs}(\textit{Cluster}, x_{t}, \ell_{m})$ method removes the stops from where the requests in the cluster would have been picked up and dropped off and then generates a new route for the robot by using the shortest path between the remaining stops. 
The method \textit{InsertReqs}(\textit{Cluster}, $\vec{A'}_{t}, \vec{y'}_{t}, \ell_{k}$) uses the insertion procedure described in algo.~\ref{algo:insertion_procedure} to insert all requests in the cluster, one request at a time, and then it updates the assignment vector and corresponding routes. 
Finally, the method \textit{GetMinControls}$(U, N^{\text{routes}})$, evaluates the stage cost $g(x_{t}, u'_t, \{r_l\})$ for every $u'_t \in U$ and returns a list of controls in ascending cost. 
If the list a larger than $N^{\text{routes}}$, then it returns the first $N^{\text{routes}}$ elements in the list (the elements with the lowest stage cost).

Now that we have all the components for our approach, we will explore the stability of our proposed base policy $\pi^{\text{greedy}}$ in the theoretical results section (see Sec.~\ref{sec:theoretical_results}), and a case study on a real transportation system in the empirical results section (see Sec.~\ref{sec:empirical_results}).

\section{Theoretical Results}
\label{sec:theoretical_results}

In this section we show that our proposed greedy policy $\pi^{\text{greedy}}$ is stable if and only if there are enough robots available to service all requests in the operating horizon $\vec{T}$ for an arbitrary day $d \in \mathcal{D}$ almost surely.
We also show that an iterative fleet sizing algorithm that increases the fleet size every time a request can't be allocated \cite{Wallar2019OptimizingVD}, but continues the assignment process using all previously determined assignments, is not guaranteed to find a sufficiently large fleet size for the stability of $\pi^{\text{greedy}}$. We then propose an algorithm for finding a sufficiently large fleet size that addresses the limitations of the previous iterative algorithm by restarting the assignment and discarding request assignments made before the fleet size was increased. We show that our fleet sizing algorithm asymptotically guarantees the stability of $\pi^{\text{greedy}}$ for any request distribution $\mathcal{P}$ as the number of historical request samples goes to infinity (i.e $|\mathcal{H}| \to \infty$). 

In this section, we call offline execution any procedure that considers exclusively the historical request data contained in $\mathcal{H}$. On the other hand, we call online execution the process of running an algorithm using a new request sequence sampled from $\mathcal{P}$.

\subsection{General Assumptions}
For all our theoretical results, we consider the following four assumptions:

\begin{assumption}
    \label{assumption:depots}
    We assume depots are distributed throughout the graph $G$ such that any point in the map is at most $W^{\text{pick}}$ seconds away from the closest depot. In other words, for every $i \in V$, we assume $\exists b \in \mathcal{B}$ such that $\textit{time}(b, i) \leq W^{\text{pick}}$.
\end{assumption}

\begin{assumption}
    \label{assumption:buffer_time}
    We assume that the time $T^{\text{last}}$ is chosen in such a way that the time difference $T^{\text{end}} - T^{\text{last}}$ is three times the time it takes to traverse the diameter of $G$.
\end{assumption}

\begin{assumption}
    \label{assumption:bounded_requests}
    We assume that the number of scheduled requests and the number of real-time requests at each time step for a given day $d$ are bounded away from infinity such that $|\mathcal{R}^{\text{sched}}| < \infty$ and $|\mathcal{R}^{\text{realtime}}_{t}| < \infty, \forall t \in \vec{T}$
\end{assumption}

\begin{assumption}
    \label{assumption:request_distribution}
    We assume that all the requests for any day $d$ are drawn from a request distribution $\mathcal{P}$, and that we have a set of historical requests sets $\mathcal{H}$ as defined in Sec.~\ref{subsec:transportation_requests} for which the requests $\mathcal{R}_{T^{\text{end}}}\in\mathcal{H}$ for any historical date $\bar{d}$ are drawn from the same request distribution $\mathcal{P}$.
\end{assumption}

Assumption~\ref{assumption:depots} is designed to prevent pathological cases where requests enter the system, but they can't be serviced by any permissible routing policy since they are too far from any depot and hence would require a robot to be already moving in that direction before the request enters the system.
Assumption~\ref{assumption:buffer_time} is designed to allow robots to pick up and drop-off requests that enter on or before $T^{\text{last}}$, while still being able to return to their original depot on or before the end of the horizon $T^{\text{end}}$. This assumption prevents pathological cases where requests can't be serviced because they have entered the system too close to the end of the horizon. The time difference $T^{\text{end}} - T^{\text{last}}$ is chosen to be three times the time it takes to traverse the diameter of the graph $G$ since that value is the worst case travel time required to service a request, where one diameter is traversed to reach the pickup of the request, one diameter is traversed to go from the pickup to the drop-off location, and a third diameter is traversed to return from the drop-off location to the original depot for the robot.
Assumption~\ref{assumption:bounded_requests} imposes a realistic constraint, assuming that a finite number of requests are scheduled a-priori and a finite number of requests enter the system at each time step.
Finally, assumption~\ref{assumption:request_distribution} guarantees that the request distribution that produced the historical data is the same as the request distribution encountered during online execution i.e. we consider "in-distribution" requests only.

Now that we have these assumptions, we can move on to our theoretical claims.

\subsection{Conditions for Stability of the Greedy Base Policy $\pi^{\text{greedy}}$}
\label{subsec:theoretical_stability_def}
In this subsection, we are interested in determining the conditions under which the base policy $\pi^{\text{greedy}}$ described in Sec.~\ref{subsec:greedy_base_policy} is stable. Specifically, we consider conditions on the fleet size as shown in the following theorem.

\begin{theorem}
    \label{theorem:stabiltiy_properties}
    Given Assumptions~\ref{assumption:depots}, \ref{assumption:buffer_time}, \ref{assumption:bounded_requests}, \ref{assumption:request_distribution}, and a fixed fleet size $|\mathcal{L}|$, our proposed greedy policy $\pi^{\text{greedy}}$ is stable as defined in Def.~\ref{def:stability} if and only if the fleet size $|\mathcal{L}|$ is large enough to allow $\pi^{\text{greedy}}$ to service all requests $\mathcal{R}_{T^{\text{end}}}$ for any day $d \in \mathcal{D}$, such that $\mathcal{R}^{\text{rejected}} = \emptyset$ almost surely.
\end{theorem}

To prove the statement of this theorem, we will first show the forward direction of the implication, demonstrating that if $\pi^{\text{greedy}}$ with a fixed fleet size $|\mathcal{L}|$ is stable according to Def.~\ref{def:stability} then the fleet size $|\mathcal{L}|$ is large enough to allow $\pi^{\text{greedy}}$ to service all requests in the operating horizon $\vec{T}$ for any day $d \in \mathcal{D}$, such that $\mathcal{R}^{\text{rejected}} = \emptyset$ almost surely. To do this, we first prove that $Q(x_{t})$ the sum of stage costs for $\pi^{\text{greedy}}$ on an arbitrary day $d \in \mathcal{D}$ starting at an arbitrary state $x_{t}$ is a non-negative random variable. We then use Markov's inequality to show that the definition of stability allows us to almost surely bound the wait times for all requests in the system for an arbitrary day $d \in \mathcal{D}$. Using the properties of our proposed greedy policy, we then show that having almost surely bounded wait times for all requests means that all requests were placed in a valid route and hence $\mathcal{R}^{\text{rejected}} = \emptyset$ almost surely. 
After showing the forward direction, we will prove the reverse direction. For the reverse direction, we want to show that if the fleet size $|\mathcal{L}|$ is large enough to allow $\pi^{\text{greedy}}$ to service all requests in the operating horizon $\vec{T}$ such that $\mathcal{R}^{\text{rejected}} = \emptyset$ for any day $d \in \mathcal{D}$ almost surely, then $\pi^{\text{greedy}}$ is stable according to Def.~\ref{def:stability}. To show this, we use the definition of valid routes given in Sec.~\ref{subsec:robot_routes} and the fact that our proposed greedy policy can only consider valid routes that do not violate the time constraints specified in Sec.~\ref{subsec:contraints_cost}. We use these two facts to show that if the fixed fleet size $|\mathcal{L}|$ is large enough for $\pi^{\text{greedy}}$ to almost surely allocate all requests to robots with valid routes, then the requests wait times are bounded away from infinity almost surely. Using the definition of almost surely, we can then easily show that almost surely bounded wait times imply a bounded policy cost. By definition of stability, we then conclude that the proposed greedy policy $\pi^{\text{greedy}}$ is stable.

\begin{proof}
    To start, we first assume that the Assumptions~\ref{assumption:depots}, \ref{assumption:buffer_time}, \ref{assumption:bounded_requests}, and \ref{assumption:request_distribution} hold and we have a fixed fleet size $|\mathcal{L}|$. For notational simplicity, we define the sum of stage costs for $\pi^{\text{greedy}}$ on an arbitrary day $d \in \mathcal{D}$ starting at an arbitrary state $x_{t}$ as 
    \begin{align*}
        Q(x_{t}) & =  \sum_{t' = t}^{T^{\text{end}}} g(x_{t'}, \mu^{\text{greedy}}_{t'}(x_{t'}), \mathcal{R}^{\text{realtime}}_{t'})
    \end{align*}
    and $J_{\pi^{\text{greedy}}} (x_{t}) = \mathds{E}[Q(x_{t})]$.

    We will first prove the forward direction of the if-and-only-if statement, showing that if $\pi^{\text{greedy}}$ is stable as defined in Def.~\ref{def:stability}, then the fleet size $|\mathcal{L}|$ is large enough to allow the greedy policy $\pi^{\text{greedy}}$ to service all requests in the operating horizon $\vec{T}$ such that $\mathcal{R}^{\text{rejected}} = \emptyset$ for an arbitrary day $d\in \mathcal{D}$ almost surely. We start by showing the non-negativity of $Q(x_{t})$. From the definition of $Q(x_{t})$, we have that: 
    \begin{align*}
        Q(x_{t}) & = \sum_{t' = t}^{T^{\text{end}}} g(x_{t'}, \mu^{\text{greedy}}_{t'}(x_{t'}), \mathcal{R}^{\text{realtime}}_{t'}) \\
        & \overset{(1)}{=} \sum_{t' = t}^{T^{\text{end}}} h(f(x_{t'}, \mu^{\text{greedy}}_{t'}(x_{t'}), \mathcal{R}^{\text{realtime}}_{t'})) -h(x_{t'}) \\
        & \overset{(2)}{=} \sum_{t' = t}^{T^{\text{end}}} h(x_{t'+1})-h(x_{t'}) \\
        & \overset{(3)}{=} \left( h\left(x_{T^{\text{end}}}\right)-h\left(x_{T^{\text{end}}-1}\right)\right) \\
        & \qquad + \left( h\left(x_{T^{\text{end}}-1}\right)-h\left(x_{T^{\text{end}}-2}\right)\right) \\
        & \qquad + \dots \\
        & \qquad + \left(h(x_{t+2})-h(x_{t+1}) \right) \\
        & \qquad + \left(h(x_{t+1})-h(x_{t}) \right) \\
        & \overset{(4)}{=} h\left(x_{T^{\text{end}}}\right) - h(x_{t})
    \end{align*}
    Where (1) follows from the definition of stage cost $g(x_{t'}, \mu^{\text{greedy}}_{t'}(x_{t'}), \mathcal{R}^{\text{realtime}}_{t'})$ in Eq.~\ref{eq:stage_cost}; (2) comes from the definition of the state transition function $f(x_{t'}, \mu^{\text{greedy}}_{t'}(x_{t'}), \mathcal{R}^{\text{realtime}}_{t'})$; (3) comes from the expansion of the summation; and (4) comes from canceling out terms in the summation. 
    Since wait times $w^{\text{pick}}_{r_q}(x_{t})$ and $w^{\text{drop}}_{r_q}(x_{t})$ are non-negative by definition for all requests $r_q$ and all states  $x_{t}$, we know that $h(x_{t}) = \sum_{r_q \in \mathcal{R}_{t}} w^{\text{pick}}_{r_q}(x_{t}) + w^{\text{drop}}_{r_q}(x_{t}) \geq 0$ for any state $x_{t}$, where $h(x_{t}) = 0$ only if no requests have entered the system before time $t$. Because the proposed greedy policy $\pi^{\text{greedy}}$ does not allow reassignment of requests, we can conclude that if a new request enters the system at time $t$, we get $h(x_{t+1}) \geq h(x_{t})$. This relationship comes from the fact that the sum of wait times for all requests in the system at time $t+1$ is equal to the sum of wait times for requests already in the system in the previous time step $t$ plus the wait times of requests that enter the system at time $t$. If no request enters the system at time $t$, then we get $h(x_{t+1}) = h(x_{t})$, since no new requests are added and hence the wait times of requests already in the system do not change under $\pi^{\text{greedy}}$. By iteratively applying the relationship $h(x_{t+1}) \geq h(x_{t})$, we obtain that $h(x_{T^{\text{end}}}) \geq h(x_{t})$, which allows us to conclude that $h\left(x_{T^{\text{end}}}\right) - h(x_{t}) \geq 0$ and hence the random variable $Q(x_{t})$ is non-negative by definition. When we consider the start of the horizon where $t=T^{\text{start}}$, we have that:
    \begin{align*}
        Q(x_{T^{\text{start}}}) & = \sum_{t' = T^{\text{start}}}^{T^{\text{end}}} g(x_{t'}, \mu^{\text{greedy}}_{t'}(x_{t'}), \mathcal{R}^{\text{realtime}}_{t'}) \\
        & \overset{(1)}{=} \sum_{t' = T^{\text{start}}+1}^{T^{\text{end}}} g(x_{t'}, \mu^{\text{greedy}}_{t'}(x_{t'}), \mathcal{R}^{\text{realtime}}_{t'}) \\
        & \qquad + g(x_{T^{\text{start}}}, \mu^{\text{greedy}}_{T^{\text{start}}}(x_{T^{\text{start}}}), \mathcal{R}^{\text{realtime}}_{T^{\text{start}}}) \\ 
        & \overset{(2)}{=} h\left(x_{T^{\text{end}}}\right) - h\left(x_{T^{\text{start}}+1}\right) + h\left(x_{T^{\text{start}}+1}\right) \\
        & =  h\left(x_{T^{\text{end}}}\right)
    \end{align*}
    Where (1) follows from separating the sum into the first term and a sum over the rest; and (2) follows from the simplification of the sum done previously for $Q(x_{t})$ and from the stage cost definition given in Eq.~\ref{eq:initial_stage_cost}. From this, we can conclude that $Q(x_{T^{\text{start}}})$ is also non-negative since $h\left(x_{T^{\text{end}}}\right) \geq 0$ by definition.

    Since for the forward direction we are assuming that $\pi^{\text{greedy}}$ with a fixed fleet size $|\mathcal{L}|$ is stable, we know from the definition of stability given in Def.~\ref{def:stability} that $J_{\pi^{\text{greedy}}}(x_{t}) = \mathds{E}[Q(x_{t})] < \infty$ for all reachable states $x_{t}$ with $t \in \vec{T}$. Since $Q(x_{t})$ is a non-negative random variable, we can apply Markov's inequality to obtain:
    \begin{align*}
        P(Q(x_{t}) \geq a) \leq \frac{\mathds{E}[Q(x_{t})]}{a}
    \end{align*}
    for $a>0$. If we take the limit as $a \to \infty$ for both sides, we get
    \begin{align*}
        \lim_{a \to \infty} P(Q(x_{t}) \geq a) & \leq \lim_{a \to \infty} \frac{\mathds{E}[Q(x_{t})]}{a} \\
        P(Q(x_{t}) \geq \infty) & \overset{(1)}{\leq} 0 
    \end{align*}
    Where inequality (1) comes from the definition of stability $\mathds{E}[Q(x_{t})] < \infty$, and hence the right-hand side goes to $0$ as $a \to \infty$. This expression allows us to conclude that the sum of stage costs is bounded away from infinity such that $Q(x_{t}) < \infty$ almost surely, i.e. the probability of $Q(x_{t}) \geq \infty$ goes to zero. Since this relationship holds for all reachable states $x_{t}$ with $t \in \vec{T}$, it must hold for the initial state $x_{T^{\text{start}}}$, which has the largest sum of stage costs. This implies that:
    \begin{align*}
        Q(x_{T^{\text{start}}}) & =  h\left(x_{T^{\text{end}}}\right) \\
        & = \sum_{r_q \in \mathcal{R}_{T^{\text{end}}}} w^{\text{pick}}_{r_q}(x_{T^{\text{end}}}) + w^{\text{drop}}_{r_q}(x_{T^{\text{end}}}) \\
        & < \infty \text{ a.s }
    \end{align*}
    From this expression, we can conclude that the wait times for each request $r_q \in \mathcal{R}_{T^{\text{end}}}$ are bounded away from infinity almost surely for an arbitrary day $d \in \mathcal{D}$. Since under $\pi^{\text{greedy}}$ the wait times for a request $r_q$ would be set to infinity if a request could not be assigned to a valid route that satisfies the time constraints specified in Sec.~\ref{subsec:contraints_cost}, we can conclude that all requests $r_q \in \mathcal{R}_{T^{\text{end}}}$ can be allocated to at least one robot $\ell_{m} \in \mathcal{L}$ under a valid route that satisfies all the wait time constraints almost surely. This implies that $\mathcal{R}^{\text{rejected}} = \emptyset$ for an arbitrary day $d \in \mathcal{D}$ almost surely, and hence the fleet size $|\mathcal{L}|$ is sufficient for servicing all requests $r_q \in \mathcal{R}_{T^{\text{end}}}$ under the policy $\pi^{\text{greedy}}$.

    We will now prove the reverse direction of the implication, showing that if the fleet size $|\mathcal{L}|$ is large enough for the greedy policy $\pi^{\text{greedy}}$ to service all requests in the operating horizon $\vec{T}$ such that $\mathcal{R}^{\text{rejected}} = \emptyset$ for an arbitrary day $d \in \mathcal{D}$ almost surely, then the policy $\pi^{\text{greedy}}$ is stable according to Def.~\ref{def:stability}. For this direction, we assume that the fleet size is large enough for $\mathcal{R}^{\text{rejected}} = \emptyset$ for an arbitrary day $d \in \mathcal{D}$ under $\pi^{\text{greedy}}$ almost surely. Since the policy $\pi^{\text{greedy}}$ only considers assignments that result in valid routes that satisfy the time constraints specified in Sec.~\ref{subsec:contraints_cost}, we can conclude that for an arbitrary state $x_{t}$ all requests are assigned to robots with valid routes that satisfy the wait time constraints such that 
    $w^{\text{pick}}_{r_q}(x_{t}) \leq W^{\text{pick}}$ and $w^{\text{drop}}_{r_q}(x_{t}) \leq W^{\text{drop}}$ for all requests $r_q \in \mathcal{R}_{t}$. Scenarios where these conditions are not satisfied have probability 0 by the definition of almost surely. For proving the implication of the reverse direction we will show that the largest cost of the policy is bounded away from infinity such that
    $J_{\pi^{\text{greedy}}}(x_{T^{\text{start}}}) < \infty$. From this result, we will use the fact that $J_{\pi^{\text{greedy}}}(x_{T^{\text{start}}}) \geq J_{\pi^{\text{greedy}}}(x_{t})$ for any state $x_{t}$ reachable under $\pi^{\text{greedy}}$ to conclude that $J_{\pi^{\text{greedy}}}(x_t) < \infty $ and hence the greedy policy $\pi^{\text{greedy}}$ is stable by definition. To prove that $J_{\pi^{\text{greedy}}}(x_{T^{\text{start}}}) < \infty$, we start with the following:
    \begin{align}
        J_{\pi^{\text{greedy}}}& (x_{T^{\text{start}}}) = \mathds{E}[Q(x_{T^{\text{start}}})] \nonumber \\
        & \overset{(1)}{=} \mathds{E}[h\left(x_{T^{\text{end}}}\right)] \nonumber \\
        & \overset{(2)}{=} \mathds{E}\left[ \sum_{r_q \in \mathcal{R}_{T^{\text{end}}}} w^{\text{pick}}_{r_q}(x_{T^{\text{end}}}) + w^{\text{drop}}_{r_q}(x_{T^{\text{end}}}) \right] \nonumber \\
        & \overset{(3)}{\leq} \mathds{E}\left[ \sum_{r_q \in \mathcal{R}_{T^{\text{end}}}}  W^{\text{pick}} + W^{\text{drop}} \right] \nonumber \\
        & \overset{(4)}{=} \left( W^{\text{pick}} + W^{\text{drop}} \right) \mathds{E}\left[ |\mathcal{R}_{T^{\text{end}}}|\right] \nonumber \\
        & \overset{(5)}{=} \left( W^{\text{pick}} + W^{\text{drop}} \right) \mathds{E}\left[ |\mathcal{R}^{\text{sched}}| + \sum_{t=T^{\text{start}}}^{T^{\text{end}}} |\mathcal{R}^{\text{realtime}}_{t}| \right] \nonumber \\
        & \overset{(6)}{=} \left( W^{\text{pick}} + W^{\text{drop}} \right) \mathds{E}[|\mathcal{R}^{\text{sched}}|] \nonumber \\
        & \qquad + \left( W^{\text{pick}} + W^{\text{drop}} \right)  \left(\sum_{t=T^{\text{start}}}^{T^{\text{end}}} \mathds{E}[|\mathcal{R}^{\text{realtime}}_{t}|] \right) \nonumber
    \end{align}
    Where (1) comes from the definition of $Q(x_{T^{\text{start}}})$; (2) follows from the definition of $h\left(x_{T^{\text{end}}}\right)$ in Eq.~\ref{eq:immediate_cost}; (3) follows from the fact that $w^{\text{pick}}_{r_q}(x_{T^{\text{end}}}) \leq W^{\text{pick}}$ and $w^{\text{drop}}_{r_q}(x_{T^{\text{end}}}) \leq W^{\text{drop}}$ almost surely for all requests $r_q \in \mathcal{R}_{T^{\text{end}}}$; equality (4) follows from the fact that $W^{\text{pick}}$ and $W^{\text{drop}}$ do not depend on $r_q$ and hence can be taken out of the sum; equality (5) follows from the definition of $|\mathcal{R}_{T^{\text{end}}}|$; and equality (6) follows from linearity of expectations. 
    Since we know from assumption~\ref{assumption:bounded_requests} that $|\mathcal{R}^{\text{sched}}| < \infty$ and $|\mathcal{R}^{\text{realtime}}_{t}| < \infty$, we can conclude that $\mathds{E}[|\mathcal{R}^{\text{sched}}|] < \infty$ and $\mathds{E}[|\mathcal{R}^{\text{realtime}}_{t}|] < \infty$. Using this result and the fact that $W^{\text{pick}}$, $W^{\text{drop}}$, $T^{\text{last}}$, and $T^{\text{start}}$ are all constants, we can conclude that $J_{\pi^{\text{greedy}}} (x_{T^{\text{start}}}) < \infty$. Since $J_{\pi^{\text{greedy}}}(x_{T^{\text{start}}}) \geq J_{\pi^{\text{greedy}}}(x_{t})$ for an arbitrary state $x_t$ due to the non-negativity of $h(x_{t})$, then we can conclude that $\infty > J_{\pi^{\text{greedy}}}(x_{T^{\text{start}}}) \geq J_{\pi^{\text{greedy}}}(x_{t})$, and hence we can conclude that $\pi^{\text{greedy}}$ is stable by definition.

    Now that we have shown both directions of the statement, we can conclude that the statement of the theorem holds.
\end{proof}

After concluding that the stability of our greedy policy $\pi^{\text{greedy}}$ depends on having a sufficiently large fleet size $|\mathcal{L}|$ to service all requests in an arbitrary day $d \in \mathcal{D}$ almost surely, we now consider potential algorithms to obtain such a desired fleet size.

\subsection{Algorithm for Estimating Fleet Size Required for Stability of Greedy Policy $\pi^{\text{greedy}}$}
\label{subsec:theoretical_fleet_size_algorithm}
We are interested in designing an algorithm that will obtain a sufficiently large fleet size offline using historical data that guarantees the stability of the proposed greedy policy $\pi^{\text{greedy}}$ when considering new request sequences sampled online. 
Inspired by different state-fo-the-art algorithms in the literature \cite{WINTER2018, Wallar2019OptimizingVD}, we first consider an offline iterative approach that increases the fleet size every time a request can't be allocated, solving the fleet sizing problem in a single-pass minimization of costs. We refer to this as the "Single-Pass algorithm". To find the fleet size using the Single-Pass algorithm (see algo.~\ref{algo:single_pass}), we consider the requests in each day inside the set of historical requests sets $\mathcal{H}$. For each day, we start the size of the fleet at 1 robot, and we start assigning requests sequentially using $\pi^{\text{greedy}}$ until we encounter a request that can't be assigned to the fleet with a valid route that satisfies the time constraints specified in Sec.~\ref{subsec:contraints_cost}. Once we encounter such a request, we increase the fleet size by one robot, and retry the assignment of the failed request. Since now we have a new robot that is free and thanks to Assumption~\ref{assumption:depots}, the new request will be allocated to the new robot, and we can continue with the assignment process as before. The final fleet size for that day is the number of robots that allowed $\pi^{\text{greedy}}$ to assign all requests using valid routes.
The final fleet size outputted by the algorithm is found by taking the maximum fleet size found while evaluating days in  $\mathcal{H}$. The Single-Pass algorithm described in algo.~\ref{algo:single_pass} uses two methods: \textit{AssignReq($r_q, x_{t}, \textit{TempSize}, \pi^{\text{greedy}}$)}, and $ \textit{ProgressState}(x_{t})$. The method \textit{AssignReq($r_q, x_{t}, \textit{TempSize}, \pi^{\text{greedy}}$)} assigns request $r_q$ to the routes contained in $x_{t}$, given a fleet size of $\textit{TempSize}$ and a policy $\pi^{\text{greedy}}$. It returns an updated state containing the new valid routes if the request can be allocated. If the request can't be allocated to a valid route, then it returns None. The method $\textit{ProgressState}(x_{t})$ moves time forward by one time-step, moving robots along their planned routes.

\begin{algorithm}
    \caption{Single-Pass Fleet Sizing Algorithm}
    \begin{algorithmic}[1]
    \label{algo:single_pass}
        \REQUIRE Set of historical request sets $\mathcal{H}$
        \ENSURE Fleet size $|\mathcal{L}|$
        \STATE $|\mathcal{L}| = 1$
        \FOR{request set $R_{T^{\text{end}}} \in \mathcal{H}$}
            \STATE \textit{TempSize} = 1
            \STATE $x_{t} = x_{T^{\text{start}}}$
            \FOR{$t \in \vec{T}$}
                \IF{$t == T^{\text{start}}$}
                    \STATE $\mathcal{R} = \mathcal{R}^{\text{realtime}}_{t} \bigcup \mathcal{R}^{\text{sched}}$
                \ELSE
                    \STATE $\mathcal{R} = \mathcal{R}^{\text{realtime}}_{t}$
                \ENDIF
                \FOR{$r_q \in \mathcal{R}$}
                    \STATE $x^{\text{temp}}_{t}$ = \textit{AssignReq($r_q, x_{t}, \textit{TempSize}, \pi^{\text{greedy}}$)} 
                    \IF{$x^{\text{temp}}_{t}$ is None}
                        \STATE $\textit{TempSize} = \textit{TempSize} + 1$
                        \STATE $x^{\text{temp}}_{t}$ = \textit{AssignReq($r_q, x_{t}, \textit{TempSize}, \pi^{\text{greedy}}$)}
                    \ENDIF
                    \STATE $x_{t} = x^{\text{temp}}_{t}$
                \ENDFOR
                \STATE $x_{t+1} = \textit{ProgressState}(x_{t})$
            \ENDFOR
            \STATE $|\mathcal{L}| = \max \{ |\mathcal{L}|, \textit{TempSize}\}$
        \ENDFOR

        \RETURN $|\mathcal{L}|$
    \end{algorithmic}
\end{algorithm}

Even though the Single-Pass algorithm is intuitive and computationally efficient, we show that it does not guarantee stability of $\pi^{\text{greedy}}$.
To show this, we consider the following additional assumption:

\begin{assumption}
    \label{assumption:termination_single_pass}
    If $|\mathcal{H}| \to \infty$, we assume that the Single-Pass algorithm terminates once it has seen all possible sequences of requests. Since the sampling space for the sequence of requests is finite, the algorithm will terminate almost surely as $|\mathcal{H}| \to \infty$.
\end{assumption}

With this assumption, we can now move on to the formal statement.

\begin{proposition}
    \label{proposition:single_pass}
    Given Assumptions~\ref{assumption:depots}, \ref{assumption:buffer_time}, \ref{assumption:bounded_requests}, \ref{assumption:request_distribution}, and \ref{assumption:termination_single_pass}, there exists a request distribution $\bar{\mathcal{P}}$ for which the fleet size $|\mathcal{L}|$ found offline using the single-pass algorithm (see algo.~\ref{algo:single_pass}) is not sufficient for stability of $\pi^{\text{greedy}}$ as defined in Def.~\ref{def:stability} during online execution under the same request request distribution $\bar{\mathcal{P}}$. Moreover, this claim holds even if $|\mathcal{H}| \to \infty$.
\end{proposition}

We prove this proposition by finding a specific request distribution for which the statement holds.

\begin{proof}
    Without loss of generality, let's consider a simplified routing set-up, where the underlying graph $G$ corresponds to a street network as shown in Fig.~\ref{fig:counter_example_diagram}. 

    \begin{figure}
    \centering
    \includegraphics[width=0.7\linewidth]{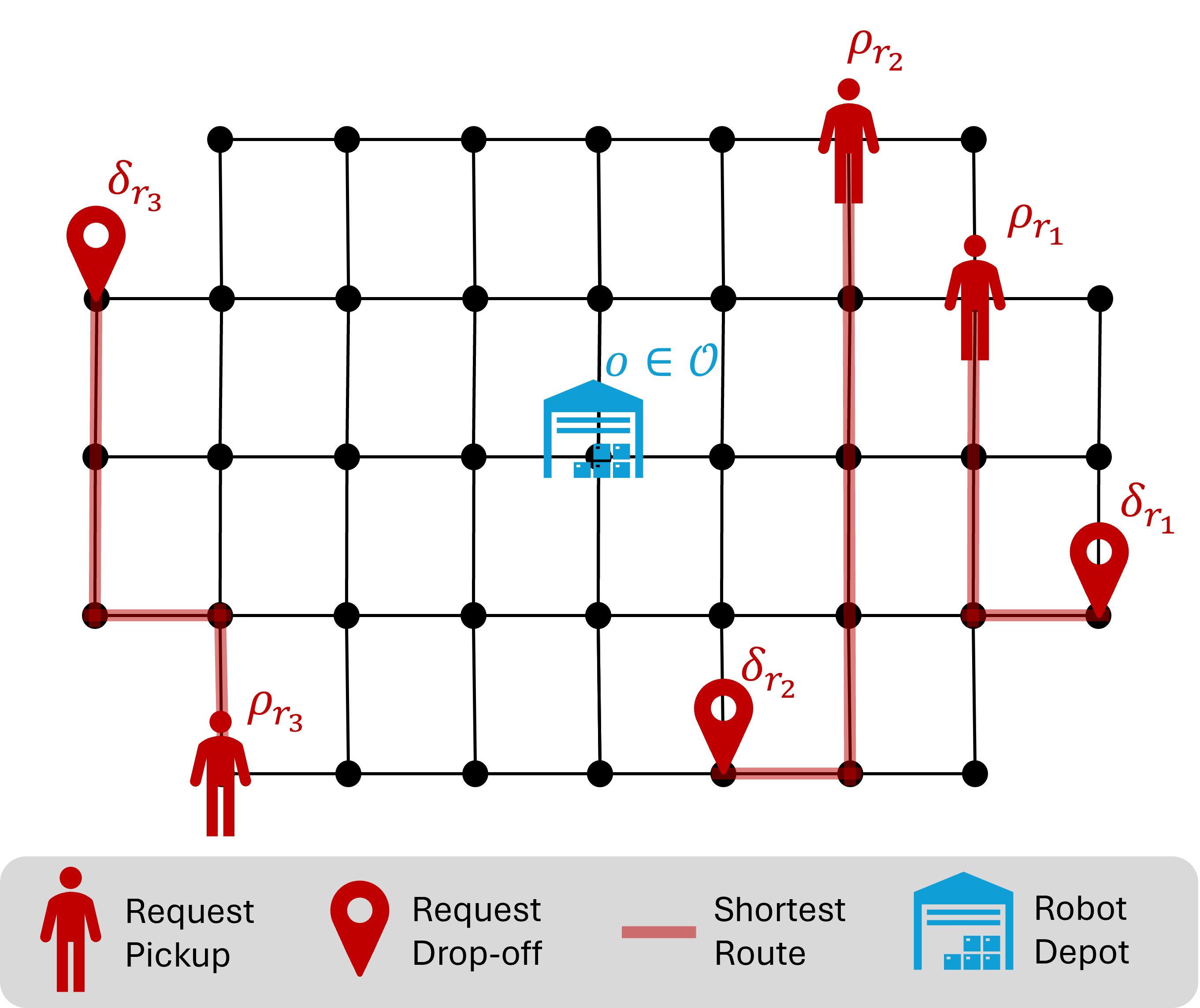}
    \caption{Simple routing example used as a counter-example that shows that the fleet size found using the single-pass algorithm is not sufficient for the stability of our proposed greedy policy $\pi^{\text{greedy}}$.}
    \label{fig:counter_example_diagram}
    \end{figure}
    
    In this set-up, all streets are assumed to be bidirectional. We assume that the time $t$ is discrete and corresponds to the number of minutes since the start of the horizon. For simplicity, we assume that it takes 1 minute to traverse any edge in the street network, there is only one depot in the middle of the network, and the wait time constraints are set to $W^{\text{pick}} = W^{\text{drop}} = 5$.
    
    Lets consider a request sequence $[r_1, r_2, r_3]$.
    The pickup and drop-off locations for the request sequence $[r_1, r_2, r_3]$ are depicted in Fig.~\ref{fig:counter_example_diagram}. In this request sequence, $r_1$ enters the system at time $t=1$, $r_2$ enters the system at time $t=2$, and $r_3$ enters the system at time $t=3$. Let's consider the request distribution $\bar{\mathcal{P}}$ for which only this request sequence is possible. Such a request distribution can be easily constructed by having all the probability mass concentrated only at the pickup and drop-off locations for the requests at the times at which each request enters. For instance, at time step $t=1$, we assign a
    pickup probability of $1.0$ in the pickup location for request $r_1$ and a drop-off probability of $1.0$ in the drop-off location for request $r_1$. All other locations have a probability mass of $0.0$. Similarly, at time step $t=2$, we assign a
    pickup probability of $1.0$ on the pickup location for request $r_2$ and a drop-off probability of $1.0$ on the drop-off location for request $r_2$. We repeat this process for all the requests. Under this request distribution $\bar{\mathcal{P}}$, all samples in $\mathcal{H}$ will be composed of the same request sequence $[r_1, r_2, r_3]$. For the case where $|\mathcal{H}| \to \infty$, this means that the algorithm will finish after observing the second sample, which is identical to the first sample. This also means that an instantiation of requests for the online execution corresponds exactly to the request sequence $[r_1, r_2, r_3]$. Since the single-pass algorithm executes an independent optimization for days in $\mathcal{H}$, it will obtain the same fleet size for days with identical request sequences. This implies that we only need to consider the output of the single-pass algorithm for the only possible request sequence $[r_1, r_2, r_3]$.
    
    Now, lets consider the assignment problem for the request sequence $[r_1, r_2, r_3]$, where we use the Single-Pass algorithm to obtain the fleet size first, and then we show that such a fleet size is insufficient for stability of the greedy policy $\pi^{\text{greedy}}$ during online assignment of requests in the request sequence $[r_1, r_2, r_3]$. With the Single-Pass algorithm, we start the fleet size with 1 robot that we denote $\ell_1$. At time $t=0$, there are no requests, so the robot remains at the depot. At time $t=1$ request $r_1$ enters the system and gets assigned to robot $\ell_1$ according to $\pi^{\text{greedy}}$. At time $t=2$, request $r_2$ enters the system and robot $\ell_1$ has moved one intersection closer to the pickup location of $r_1$. There is only one stop sequence that produces a valid route according to $\pi^{\text{greedy}}$, $\vec{S}_1 = [\rho_{r_1}, \rho_{r_2}, \delta_{r_1}, \delta_{r_2}]$. The other three possible stop sequences $\vec{S}_2 = [\rho_{r_1}, \rho_{r_2}, \delta_{r_2}, \delta_{r_1}]$, $\vec{S}_3 = [\rho_{r_2}, \rho_{r_1}, \delta_{r_1}, \delta_{r_2}]$ and $\vec{S}_4 = [\rho_{r_2}, \rho_{r_1}, \delta_{r_2}, \delta_{r_1}]$ violate the wait time constraint for request $r_1$ as picking up $r_2$ first will result in $w^{\text{pick}}_{r_1} > 5$, and dropping off $r_2$ before dropping off $r_1$ will result in $w^{\text{drop}}_{r_1} > 5$. For this reason, $\pi^{\text{greedy}}$ will select $\vec{S}_1$ and assign $r_2$ to $\ell_1$. At time $t=3$, request $r_3$ enters the system and robot $\ell_1$ has moved two intersections towards the pickup location of $r_1$. At this point, robot $\ell_1$ is too far to produce valid routes for request $r_3$, and hence a new robot $\ell_2$ gets added into the system. Request $r_3$ gets assigned to $\ell_2$ and the final fleet size that is output by the Single-Pass algorithm is 2 robots since there are no more requests entering the system after $t=3$.

    Now, if we consider the same request sequence $[r_1, r_2, r_3]$ during online execution, and evaluate the greedy policy using a fixed fleet size of 2 robots, we get the following behavior. At time $t=1$, request $r_1$ enters the system and gets assigned to robot $\ell_1$. At time $t=2$, robot $\ell_1$ has moved one intersection towards request $r_1$, robot $\ell_2$ remains at the depot, and request $r_2$ enters the system. According to $\pi^{\text{greedy}}$ request $r_2$ encounters two valid routes, in the first one, it gets added to robot $\ell_1$, while in the second one, it gets added to robot $\ell_2$. Since $\pi^{\text{greedy}}$ minimizes the wait times for all requests in the routes, it will select the second route and request $r_2$ will get assigned to robot $\ell_2$. At time $t=3$, robot $\ell_1$ has moved two intersections towards request $r_1$, robot $\ell_2$ has moved one intersection towards request $r_2$, and request $r_3$ enters the system. At this point, both robots are too far to produce any valid routes for the assignment of $r_3$. This means that request $r_3$ can't be assigned to the fleet and hence it will be added to $\mathcal{R}^{\text{rejected}}$. Since the request sequence $[r_1, r_2, r_3]$ has a non-zero probability of occurring under $\bar{\mathcal{P}}$, we know that the fleet size $|\mathcal{L}|$ found offline using the Single-pass algorithm is not sufficiently large to allow $\pi^{\text{greedy}}$ to service all requests $\mathcal{R}_{T^{\text{end}}}$ for any day $d \in \mathcal{D}$, such that $\mathcal{R}^{\text{rejected}} = \emptyset$ almost surely. From Thm.~\ref{theorem:stabiltiy_properties}, we can conclude that the policy $\pi^{\text{greedy}}$ is not stable with the fleet size found offline $|\mathcal{L}|$.
\end{proof}

\begin{remark}
    The Single-Pass fleet-sizing algorithm fails at finding a sufficiently large fleet size for the stability of $\pi^{\text{greedy}}$ since it increases the fleet-size dynamically without considering the allocations that the policy would have executed in the previous assignments if more robots were available. This discrepancy in the way in which the fleet is constructed and how the robots are utilized during online execution of $\pi^{\text{greedy}}$ allows for the existence of sequences of requests for which the fleet-size found offline is not guaranteed to produce a stable policy assignment even if the same request sequence is observed during online execution. Even though we only consider a single request sequence in the proof, it is also important to note there are other possible request sequences that will result in instability of $\pi^{\text{greedy}}$ during online execution. These request sequences are characterized by requests being temporally close, but spatially distributed at opposite ends of the area covered by the depot.
\end{remark}

In order to guarantee the stability of $\pi^{\text{greedy}}$ we propose an improved version of algo.~\ref{algo:single_pass}, allowing for the new algorithm to restart the optimization process every time the fleet size increases. We call this algorithm the Restart-and-Optimize algorithm, and we describe it in algo.~\ref{algo:fleet_size}. The Restart-and-Optimize algorithm described in algo.~\ref{algo:fleet_size} uses the same two methods as the Single-Pass algorithm: \textit{AssignReq($r_q, x_{t}, \textit{TempSize}, \pi^{\text{greedy}}$)} , and $ \textit{ProgressState}(x_{t})$. The main difference between the two algorithms, is that every time a request can't be allocated in the Restart-and-Optimize algorithm, instead of increasing the fleet size and allocating the request to the new robot, the algorithm increases the fleet size and restarts the assignment process from the start of the time horizon. This restart allows the algorithm to obtain a fleet size that will match the available resources expected during online execution. In the theorem below, we show that our proposed Restart-and-Optimize algorithm finds a sufficiently large fleet size that asymptotically guarantees stability of $\pi^{\text{greedy}}$ as the number of historical request samples goes to infinity (i.e $|\mathcal{H}| \to \infty$), overcoming the limitations of the single-pass algorithm.

\begin{algorithm}
    \caption{Restart-and-Optimize Fleet Sizing Algorithm}
    \begin{algorithmic}[1]
    \label{algo:fleet_size}
        \REQUIRE Set of historical requests sets $\mathcal{H}$, maximum number of robots to be considered $M^{\text{max}}$
        \ENSURE Fleet size $|\mathcal{L}|$
        \STATE $|\mathcal{L}| = 1$
        \FOR{request set $R_{T^{\text{end}}} \in \mathcal{H}$}
            \FOR{$\textit{TempSize} \in [1, \dots, M^{\text{max}}]$}
                \STATE $x_{t} = x_{T^{\text{start}}}$
                \STATE \textit{ValidFlag} = True
                \FOR{$t \in \vec{T}$}
                    \IF{$t == T^{\text{start}}$}
                        \STATE $\mathcal{R} = \mathcal{R}^{\text{realtime}}_{t} \bigcup \mathcal{R}^{\text{sched}}$
                    \ELSE
                        \STATE $\mathcal{R} = \mathcal{R}^{\text{realtime}}_{t}$
                    \ENDIF
                    \FOR{$r_q \in \mathcal{R}$}
                        \STATE $x^{\text{temp}}_{t}$ = \textit{AssignReq($r_q, x_{t}, \textit{TempSize}, \pi^{\text{greedy}}$)} 
                        \IF{$x^{\text{temp}}_{t}$ is None}
                            \STATE \textit{ValidFlag} = False
                            \STATE Break Out of For Loop
                        \ENDIF
                        \STATE $x_{t} = x^{\text{temp}}_{t}$
                    \ENDFOR
                    \IF{\textit{ValidFlag} is False}
                        \STATE Break Out of For Loop
                    \ENDIF
                    \STATE $x_{t+1} = \textit{ProgressState}(x_{t})$
                \ENDFOR
                \IF{\textit{ValidFlag} is True}
                    \STATE $|\mathcal{L}| = \max \{ |\mathcal{L}|, \textit{TempSize}\}$
                    \STATE Break Out of For Loop
                \ENDIF
            \ENDFOR
        \ENDFOR
        \RETURN $|\mathcal{L}|$
    \end{algorithmic}
\end{algorithm}

To prove that this algorithm produces a sufficiently large fleet size for the stability of $\pi^{\text{greedy}}$, we consider the following additional assumption: 
\begin{assumption}
    \label{assumption:restart_and_optimize_termination}
    If $|\mathcal{H}| \to \infty$, we assume that the Restart-and-Optimize algorithm terminates once it has seen all possible sequences of requests. Since the sampling space for the sequences of requests is finite, the algorithm will terminate almost surely as $|\mathcal{H}| \to \infty$. 
\end{assumption}

With this assumption, we can now move on to the formal statement.

\begin{theorem}
    Given Assumptions~\ref{assumption:depots}, \ref{assumption:buffer_time}, \ref{assumption:bounded_requests}, \ref{assumption:request_distribution}, and \ref{assumption:restart_and_optimize_termination}, for any request distribution $\mathcal{P}$, if $|\mathcal{H}| \to \infty$, the fleet size $|\mathcal{L}|$ found offline using the Restart-and-Optimize algorithm (see algo.~\ref{algo:fleet_size}) is guaranteed to be sufficiently large for stability of $\pi^{\text{greedy}}$ as defined in Def.~\ref{def:stability} during online execution under the same distribution $\mathcal{P}$.
\end{theorem}

\begin{proof}
    We will prove the theorem by contradiction.
    Let's consider an arbitrary request distribution $\mathcal{P}$. Since the graph $G$ is finite, the set of potential locations where requests can be picked-up and dropped-off is also finite. Using this observation and the fact that we assume that the number of incoming requests is finite (i.e. $|\mathcal{R}^{\text{realtime}}_{t}| < \infty, \forall t$), we can conclude that the support for the request distribution $\mathcal{P}$ is finite, and hence there are a finite number of potential sequences of requests that can be sampled from request distribution $\mathcal{P}$. If we take the number of historical dates to infinity (i.e. $|\mathcal{H}| \to \infty$), we can conclude that $\mathcal{H}$ will include all potential request sequences possible almost surely, and hence the fleet size outputted by our algorithm will be the maximum fleet-size found after considering all possible request sequences. Since $\mathcal{P}$ is also the request distribution used during online execution, we know that all possible sequences of online requests are also contained inside the historical requests sets in $\mathcal{H}$.
    
    Let's denote the date in $\mathcal{H}$ that produced the maximum fleet size according to the Restart-and-Optimize algorithm as $d_{max}$ and the set of requests associated with it as $\mathcal{R}_{d_{max}} \in \mathcal{H}$. Now let's consider an arbitrary day $\bar{d}$. For the sake of the contradiction, let's assume that the fleet size $|\mathcal{L}|$ found by the Restart-and-Optimize algorithm is not sufficiently large to guarantee stability for $\pi^{\text{greedy}}$ during online execution. From Theorem \ref{theorem:stabiltiy_properties}, we know that $\pi^{\text{greedy}}$ is stable if and only if $\mathcal{R}^{\text{rejected}} = \emptyset$ for any day $d \in \mathcal{D}$ almost surely. So the fact that $\pi^{\text{greedy}}$ is not stable implies that there exists an arbitrary day $\bar{d}$ that has a sequence of requests with non-zero probability for which $|\mathcal{R}^{\text{rejected}}|>0$. 
    Since all dates in $\mathcal{H}$ are sampled from $\mathcal{P}$, we know that all request sequences with a non-zero probability are contained in $\mathcal{H}$. We will consider two cases for the arbitrary day $\bar{d}$ for which the policy is unstable during online execution: 1) $\bar{d}=d_{max}$, 2) $\bar{d}\neq d_{max}$. For the case $\bar{d}=d_{max}$, we know that obtaining $|\mathcal{R}^{\text{rejected}}|>0$ during online execution is not possible, since we know that the request sequence for day $d_{max}$ is also contained in $\mathcal{H}$, and hence the fleet size $|\mathcal{L}|$ is sufficiently large to guarantee stability of $\pi^{\text{greedy}}$ during online execution by construction. If $|\mathcal{L}|$ was not sufficiently large the Restart-and-Optimize algorithm would have increased the fleet size and restarted the assignment process for that day until a sufficiently large fleet size was found. For the case $\bar{d}\neq d_{max}$, we know that the request sequence for day $\bar{d}$ was also contained in $\mathcal{H}$ and since $\bar{d}\neq d_{max}$ the fleet size found for day $\bar{d}$ is smaller than the fleet size found for day $d_{max}$. Since the fleet size for day $\bar{d}$ is smaller than the final fleet size output by the Restart-and-Optimize algorithm (the fleet size obtained for day $d_{max}$), we also know that it is not possible for $|\mathcal{R}^{\text{rejected}}|>0$ during online execution. This is the case since having a larger fleet size than the one found offline for $\bar{d}$ allows requests that would have been allocated to the same robot to be distributed across more robots, reducing wait-times. 
    Now that we have checked the two cases, we can conclude that there is a contradiction, and hence if $|\mathcal{H}| \to \infty$, the fleet size $|\mathcal{L}|$ found offline using the Restart-and-Optimize algorithm is sufficiently large to guarantee stability for $\pi^{\text{greedy}}$ during online execution.
\end{proof}

\begin{remark}
    It is important to note that it is not necessary to take the number of historical days to infinity (i.e. $|\mathcal{H}| \to \infty$) to find a sufficiently large fleet size using the proposed Restart-and-Optimize algorithm in practice. As long as the historical data in $\mathcal{H}$ is representative of the expected demand during online execution of the system, even a finite amount of samples will allow the algorithm to obtain a sufficiently large fleet size $|\mathcal{L}|$ for stability of $\pi^{\text{greedy}}$. Characterizing the exact number of samples required for the algorithm to output a sufficiently large fleet size is an open problem; however, we empirically demonstrate that around $200$ request sequences ($200$ days of requests) for $\mathcal{H}$ is enough to find a sufficiently large fleet size that empirically maintains stability of $\pi^{\text{greedy}}$ during online execution for a minibus routing application. We include these results in Sec.~\ref{subsec:empirical_results_stability}.
\end{remark}

In the following section we will explore a case study to validate the practical applications of our theoretical results and to demonstrate the advantage of our proposed method.

\section{Case Study and Empirical Results: Harvard's Evening Van}
\label{sec:empirical_results}

To evaluate our proposed method, we consider a real multi-capacity multi-robot routing setup. In particular, we consider historical data for Harvard's Evening Van service, a university-wide minibus transportation service that allows students to book rides from any intersection in Harvard's campus to any other intersection within the operational area (see Fig.~\ref{fig:street_network}). Requests for this service can be either booked in advance or placed in real-time for immediate pickup. We test our approach using historical evening van data for $8$ months in $2022$ and $3$ months in $2023$ with a total of $254$ days worth of data, which contained around $40,000$ transportation requests with $177$ requests per day on average. We divide the evening van historical data into three non-overlapping datasets: training, testing, and validation. For the testing set, we randomly select two days for each month in the evening data and remove these dates from the historical data pool. We then randomly select one day for each month in the remaining data pool to create the validation data. The training dataset then corresponds to all the remaining days in the data pool after the dates in the validation and testing sets have been removed. Using this procedure, we obtain $20$ test days, $10$ validation days, and $224$ training days. 

The evening van service area considered in this case study (see Fig.~\ref{fig:street_network}) is composed of $833$ intersections and $1820$ streets. We consider that each time step in the system corresponds to $1$ second, and the operational times of the evening van for a given day $d$ are from $T^{\text{start}} = 7pm$ to $T^{\text{end}} = 3am$, such that $T = 28800$ (there are $28800$ seconds in $8$ hours). We determine travel times between adjacent intersections by dividing the length of each street segment by its associated posted speed limit.
For the map shown in Fig.~\ref{fig:street_network}, there is only one depot. Based on this constraint, we choose the maximum wait time at the station and the maximum wait time inside the bus to be equal to $15$ minutes ($900$ seconds), such that $W^{\text{pick}} = W^{\text{drop}} = 900$. With these maximum wait times, any point in the map can be reached by a bus dispatched from the depot without violating these constraints. Each minibus in the evening van fleet has a maximum capacity of $C=16$ passengers. For our stability experiments, we vary the number of minibuses in the evening van fleet from $2$ to $5$. For the comparison experiments, however, we consider the real fleet size $M = 3$ currently being used by the Evening Van Service.

\begin{figure}
    \centering
    \includegraphics[width=0.6\linewidth]{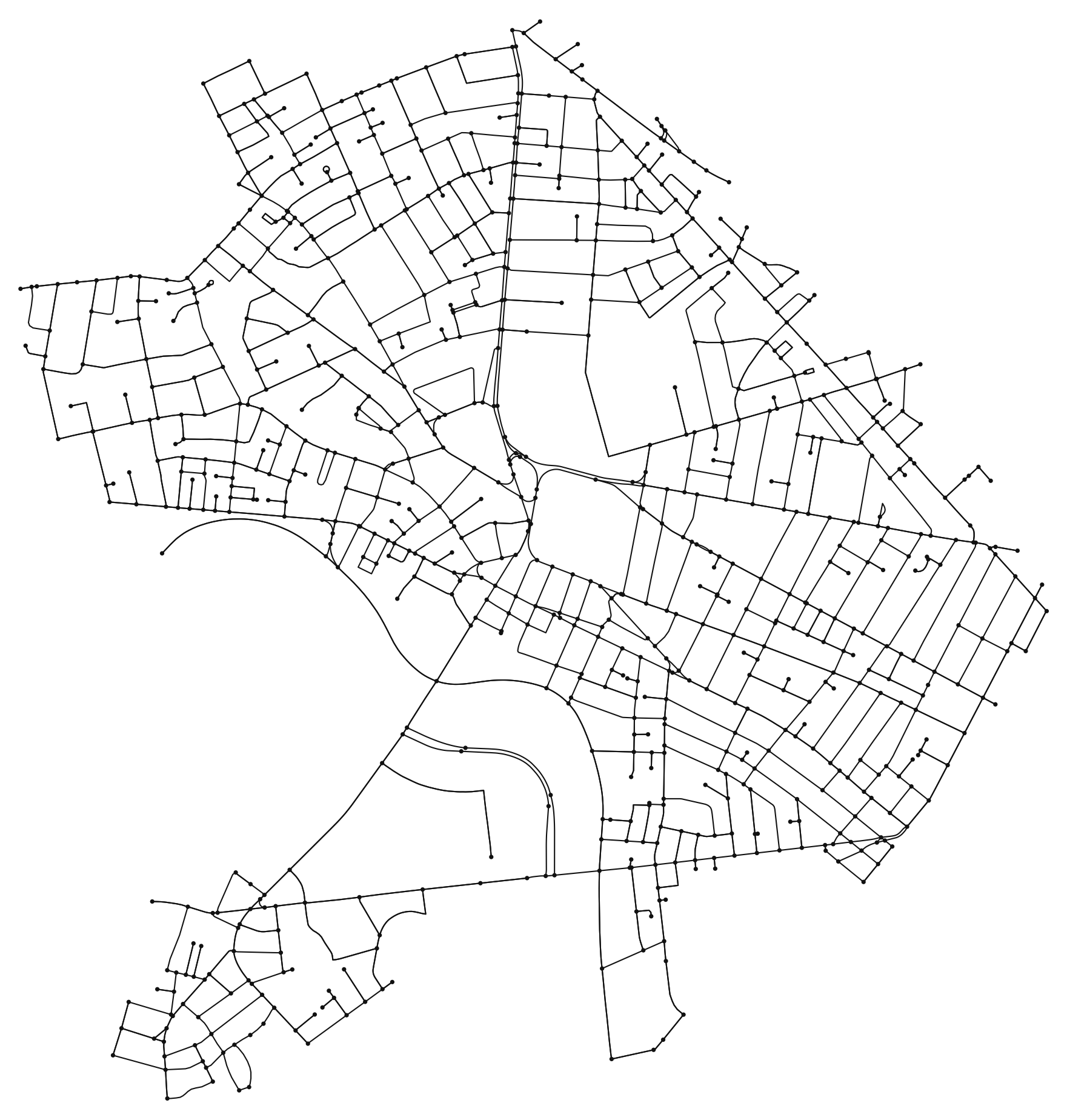}
    \caption{Street network of the operational area for Harvard's Evening Van system.}
    \label{fig:street_network}
\end{figure}

In the following subsections, we will empirically verify our theoretical results, and we will provide more details about implementation choices for our method, the baselines against which we compare, and comparative performance results of our proposed approach.

\subsection{Implementation Details for Our Approach}
To implement our proposed system, we need to provide additional details about the user-defined parameter values for the one-request-at-a-time algorithm, the generative model, and the promising-routes generator.

\subsubsection{One-request-at-a-time algorithm}
We set the number of time-steps to be considered in the application of the base policy for future cost estimation to be $K=3600$ to consider potential future scenarios for the next hour of operation.

\subsubsection{Generative model for future requests prediction}
\label{subsubsec:implementation_details_generative_model}
For the time-series transformer used for the prediction of the number of requests, we set the number of predictive intervals in an hour to be $N=12$. This choice implies that a given interval $l$ for a given hour $h$ will have a size of  $|I^{l}_{h}| = 300$, and hence the system predicts the number of requests inside $5$-minute intervals for the next hour $h+1$. Weather information for the context variables is obtained from Open-Meteo Historical Weather Data API\cite{Open-Meteo}, which uses ERA5\cite{ERA5, ERA5-land}. For the transformer itself, we consider four different architectures TimesNet\cite{wu2023timesnet}, PatchTST \cite{nie2023a}, TimeMixer \cite{wang2024timemixer}, and iTransformer \cite{liu2024itransformer}. We selected these four architectures due to their recent success on other time-series prediction tasks.
We train all four architectures for $300$ epochs using the training set. We select a mean average error (MAE) as the loss function, and we use Adam \cite{Kingma2014AdamAM} as the optimizer. The learning rate for each architecture is chosen using a hyper-parameter search. We choose a batch size of 32, and we shuffle the batches at every epoch during training. The number of layers and the width of each layer for each architecture are chosen following the recommendations given in their respective papers. All training was done using two NVIDIA RTX A6000 ADA. 
Based on prediction results for the validation set, we chose TimeMixer as the architecture of choice for the routing setting since it resulted in the lowest prediction error. 

We use the training dataset to estimate the conditional distributions for the pickup and drop-off locations. These distributions are calculated offline and stored for sampling during the rollout execution. For the process of sampling potential future requests, we choose the number of future scenarios to be $N^{\text{scenarios}} = 20$. This number was chosen empirically based on the observed trade-off between computational complexity and performance. A higher number of potential future scenarios will result in a more accurate estimation of expected future costs, but will increase the runtime of the algorithm.

\subsubsection{Promising-routes generator}
To implement the promising-routes generator, we need to specify the maximum number of controls to be considered in the final promising-routes set $N^{\text{routes}}$, and the clustering size for the clustering algorithm. 
For this setup, we chose $N^{\text{routes}} = 15$ and the minimum clustering size for the HDBSCAN clustering algorithm to be $2$. We selected $15$ as the number of routes to explore the action space enough to have some diversity of routes, but still maintain it sufficiently small that the minimization over the routes can be executed in a reasonable time. We chose the HDBSCAN algorithm \cite{Campello2015} \cite{Malzer2020} since it performs clustering in a hierarchical manner and hence the only hyper parameter needed is the minimum clustering size. We selected the minimum clustering size to be $2$ to allow for pair-wise groupings of requests that are sufficiently close, while still allowing for larger groupings to emerge. 

\subsection{Empirical Validation of Theoretical Results for Stability}
\label{subsec:empirical_results_stability}
In this section, we present our empirical studies to validate our theoretical claim that our Restart-and-Optimize fleet-sizing algorithm (see algo.~\ref{algo:fleet_size}) can find a sufficiently large fleet size that maintains stability of the proposed greedy base policy $\pi^{\text{greedy}}$ during online execution. To do this, we consider the set of $224$ days contained in the training dataset as the set of historical dates $\mathcal{H}$. We set the maximum number of robots to be considered in our algorithm to be $M^{\text{max}}=100$. After we run our proposed Restart-and-Optimize fleet-sizing algorithm on $\mathcal{H}$, we obtain that $5$ minibuses should be enough to guarantee stability of our proposed greedy base policy. We then consider the $20$ dates contained in the testing set as the unseen samples for online execution. We execute both the greedy base policy as well as our proposed one-request-at-a-time rollout algorithm for different fleet sizes $(2, 3, 4, 5)$, and count the number of rejected requests for all $20$ testing dates. We present the results for these stability studies in Table~\ref{tab:stability_results} below. As we can see in the table, a fleet size of $5$ is sufficiently large for all requests to be serviced, and hence we can conclude that both the greedy base policy and the learned rollout policy are stable (see Thm.~\ref{theorem:stabiltiy_properties}). This empirically demonstrates that using historical data, even with a finite number of dates, allows our algorithm to find a sufficiently large fleet size for both the base policy and our proposed rollout policy to service all requests in practice.

\begin{table}[b!]
    \centering
    \caption{Empirical Verification of Our Theoretical Results Using Harvard's Evening Van Requests' Data}
    \begin{tabular}{|c||c|c|} \hline
    \multirow{2}{*}{Number of Robots} & \multicolumn{2}{c|}{Number of Rejected Requests For:} \\ \cline{2-3}
    & Greedy Base Policy & Our Approach \\ \hline
    2 & 363 & 260 \\ \hline
    3 & 20 & 10 \\ \hline
    4 & 0 & 0 \\ \hline
    5 & 0 & 0 \\ \hline
    \end{tabular}
    \label{tab:stability_results}
\end{table}

\subsection{Comparative Results}
In this section, we present a comparative study of our proposed framework and multiple baselines from the literature (see Sec.~\ref{subsec:benchmarks} below). We evaluate all policies using the test dataset described in Sec.\ref{subsubsec:implementation_details_generative_model}, where the real evening van requests are the ground truth data. For this section of the study, we assume that we have $3$ minibuses in the fleet and a single depot to match the real set-up being currently used by the Evening Van Service. This choice of fleet-size, however, means that we are below the sufficiently large fleet-size that we found in our theoretical results, and hence some requests may be rejected for all policies considered. Under this setup, our objective then becomes to minimize the wait times of all requests while decreasing the number of rejected requests over all testing dates.

\subsubsection{Benchmarks}
\label{subsec:benchmarks}
Our comparison results compare the performance of our approach to the performance of six different baselines:
\paragraph{Current Evening Van} This method is the current routing algorithm deployed in the evening van system. The performance of this method is extracted from the evening van historical data, which contains information about the real wait times of the requests and the trip lengths. This method might have higher wait times and a higher number of rejected requests due to the fact that it has to deal with real traffic conditions that affect how fast the vehicles can traverse the graph.
\paragraph{Greedy PTT} This method corresponds to the greedy policy described in Sec.~\ref{subsec:greedy_base_policy}, but instead of the stage cost it uses the Passenger Travel Time (PTT) cost described in \cite{Wilbur2022}.
\paragraph{MCTS PTT} This method is inspired in the Monte-Carlo Tree Search (MCTS) approach introduced in \cite{Wilbur2022} and uses their PTT cost. The algorithm for promising action generation and bootstrapping for sampling future requests are preserved. We increase the number of promising actions to be considered to $15$ to match our maximum number of routes and provide an equitable comparison. We modify the MCTS implementation for exploration of the assignment space to allow requests to have a zero lead time (requests can be placed into the system just moments before their desired pickup time). This is achieved by limiting the swaps of requests between robots to only requests that haven't been picked up. We set the number of MCTS simulations to $1000$ and the tree depth to $20$ to balance performance and runtime, as suggested in the paper \cite{Wilbur2022}.
\paragraph{Greedy Budget} This method corresponds to the greedy policy described in Sec.~\ref{subsec:greedy_base_policy}, but instead of the stage cost it uses the budget cost described in \cite{Wilbur2022}.
\paragraph{MCTS Budget} This method has the same implementation as MCTS PTT described before, but it uses the budget cost introduced in \cite{Wilbur2022}.
\paragraph{Greedy WT} This method corresponds to our proposed greedy policy described in Sec.~\ref{subsec:greedy_base_policy}, using the stage cost defined Sec.~\ref{subsec:stability_definition}.

\subsubsection{Performance results}
\label{subsubsec:performance_results}
To evaluate the performance of our proposed algorithm as compared to the baselines described in Sec.~\ref{subsec:benchmarks}, we report three main metrics, the average wait time per passenger before being picked up (in seconds), the average trip length for serviced passengers (in seconds), and the percentage of rejected requests. The average wait time per passenger is calculated by averaging passenger's wait times for each of the $20$ dates in the testing set, and then reporting the average for each day. Similarly, the average trip length per serviced passenger is calculated by averaging passengers' trip length for each of the $20$ dates in the testing set, and reporting the average for each day. Finally, The percentage of rejected requests is calculated by dividing the number of rejected requests for each day in the test set by the total number of requests that enter the system in that day, and then multiplying that number by $100$. Results for each day in the test set are reported. We present all results for the three metrics in Fig.~\ref{fig:comparison_results}. 

\begin{figure*}[ht]
    \centering
    \includegraphics[width=\textwidth]{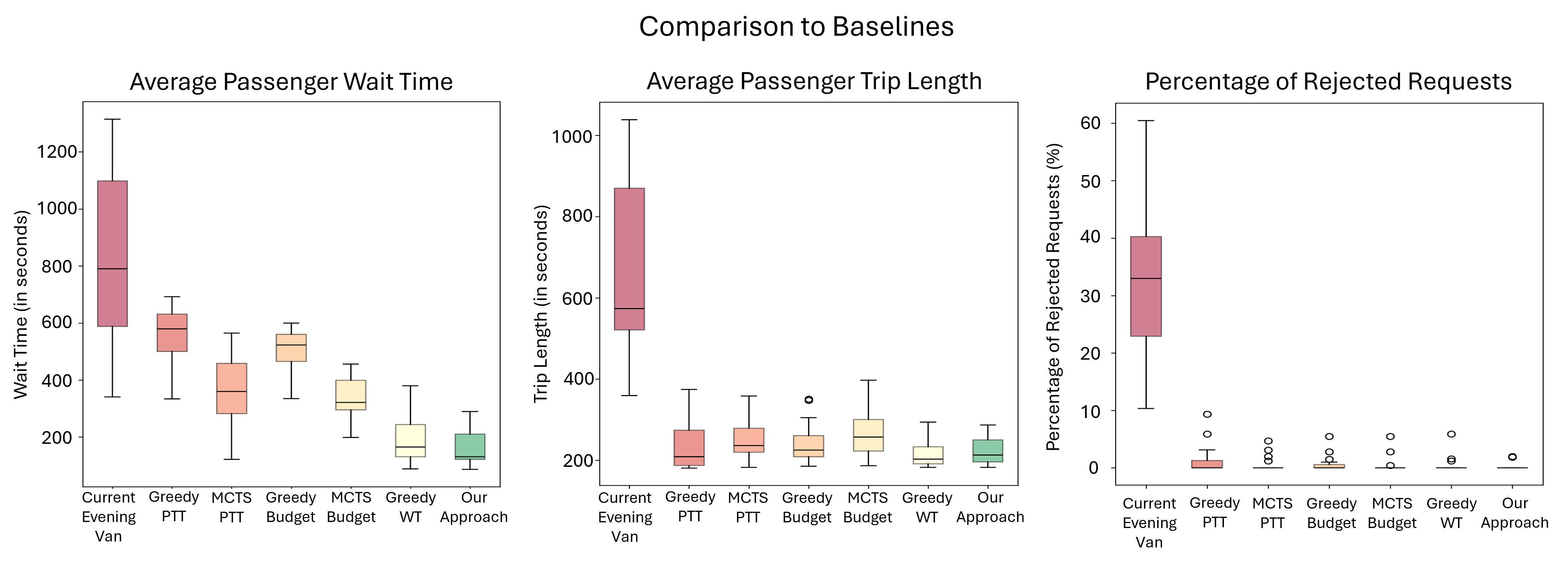}
    \caption{\small{Box plots for the average wait time per request, average trip length per request, and the percentage of rejected requests for all baselines presented in Sec.~\ref{subsec:benchmarks} and our approach.}}
    \label{fig:comparison_results}
\end{figure*}

As shown in Fig.~\ref{fig:comparison_results}, our proposed framework decreases the maximum percentage of rejected requests by $6\%$, obtaining a percentage of rejected requests of $3\%$ compared to $9\%$ for the closest baseline. Our approach also decreases the median average passenger wait time at the station to around $2$ minutes while all the other baselines have median average passenger wait times of $3$ to $12$ minutes. In other words, our method is able to service around $6\%$ more requests than the closest baseline, decrease median passenger wait times at the station by almost one minute, and still maintain similar trip lengths compared to all the other baselines. 

\subsubsection{Runtime Results}
We also present results for the planning time required before a request gets assigned for all the baselines described in Sec.~\ref{subsec:benchmarks}. We do not include the runtime for the currently implemented algorithm in the evening van system, as we do not have access to that information in the data set. We present these results in Fig.~\ref{fig:runtime_results}. 

\begin{figure}
    \centering
    \includegraphics[width=0.75\linewidth]{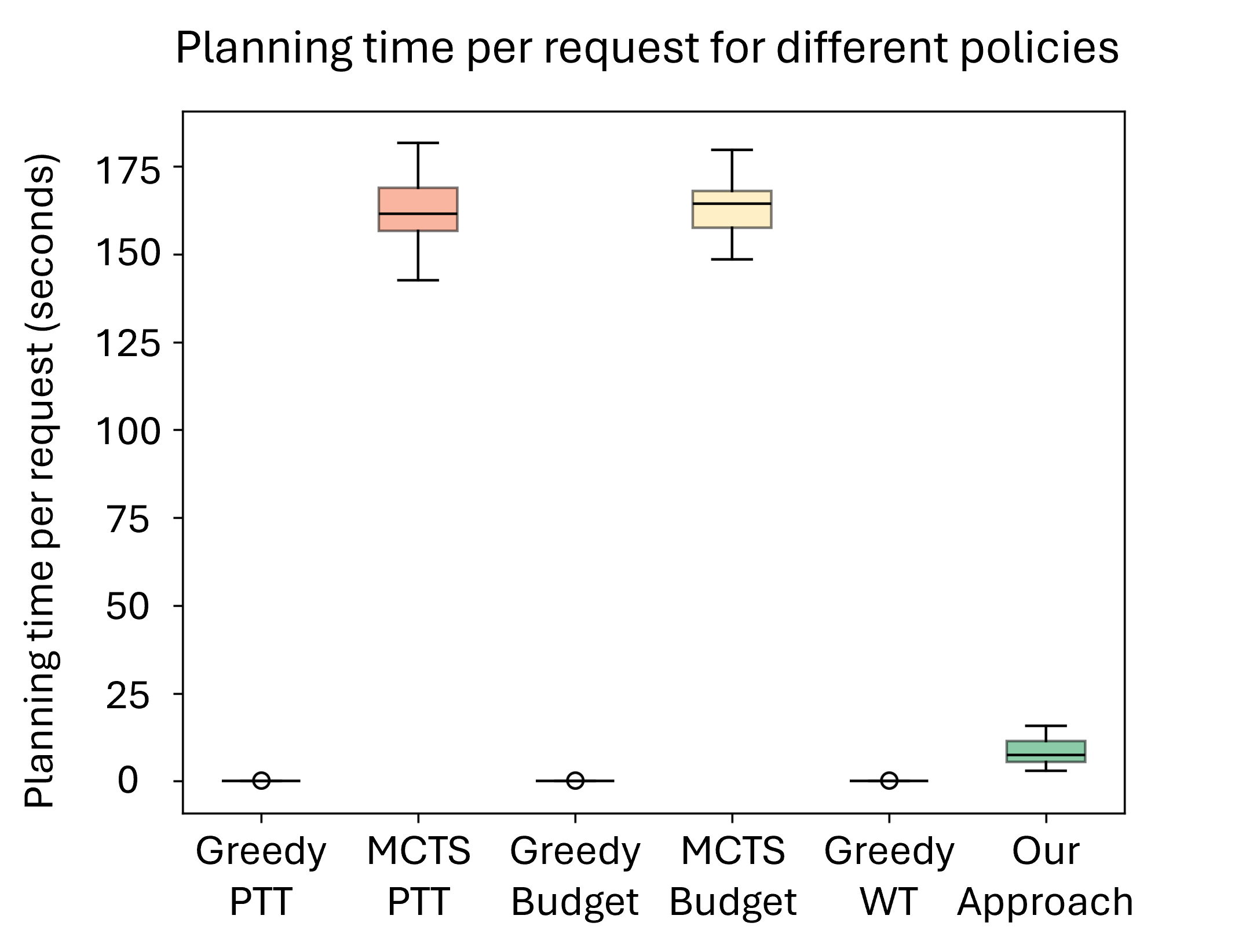}
    \caption{\small{Box plot for the planning time required to generate a routing plan for one request (in seconds).}}
    \label{fig:runtime_results}
\end{figure}

As we can see from Fig.~\ref{fig:runtime_results}, our algorithm has larger planning times than the myopic greedy algorithms, since our algorithm considers potential future requests in its decision making framework, which increases the time required for planning. Compared to other algorithms (MCTS-based) that consider future requests, our algorithm is around $8$ times faster, generating routes for a request assignment in less than $20$ seconds, compared to $170$ seconds for the other baselines that consider future requests. 

\subsection{Ablation Studies}
\label{subsec:ablation_results}
In this section, we report results that evaluate the effect of using different generative models and promising-routes generators in our system. We cover the results of these ablation studies in the following subsections.

\subsubsection{Changing the Generative Model}
\label{subsubsec:ablation_results_generative_model}
To measure the effect of different generative models for generating future scenarios, we consider three different ablation cases:
\paragraph{Bootstrapping} This method is based on the request sampling approach proposed in \cite{Wilbur2022}. This method fits a Gaussian model to the training data to obtain the number of requests per day. It samples the fitted Gaussian model to obtain the number of requests in a day. Given this number, it then uses the requests in the training dataset as a request bank and samples uniformly as many requests as dictated by the sampled number of requests obtained from the Gaussian model.
\paragraph{Our approach} this corresponds to our proposed generative model described in Sec.~\ref{subsec:generative_model} of this paper.
\paragraph{Perfect Prediction} This method uses the ground truth data as the prediction for the future requests. In this case, this method acts as an oracle and hence an upper bound in terms of the best achievable performance for our routing framework.

For this ablation study, we report the same three metrics as in the comparison studies (see Sec.~\ref{subsubsec:performance_results}). The results for this ablation are contained in Fig.~\ref{fig:generative_model_ablation}. 

\begin{figure*}[ht]
    \centering
    \includegraphics[width=\textwidth]{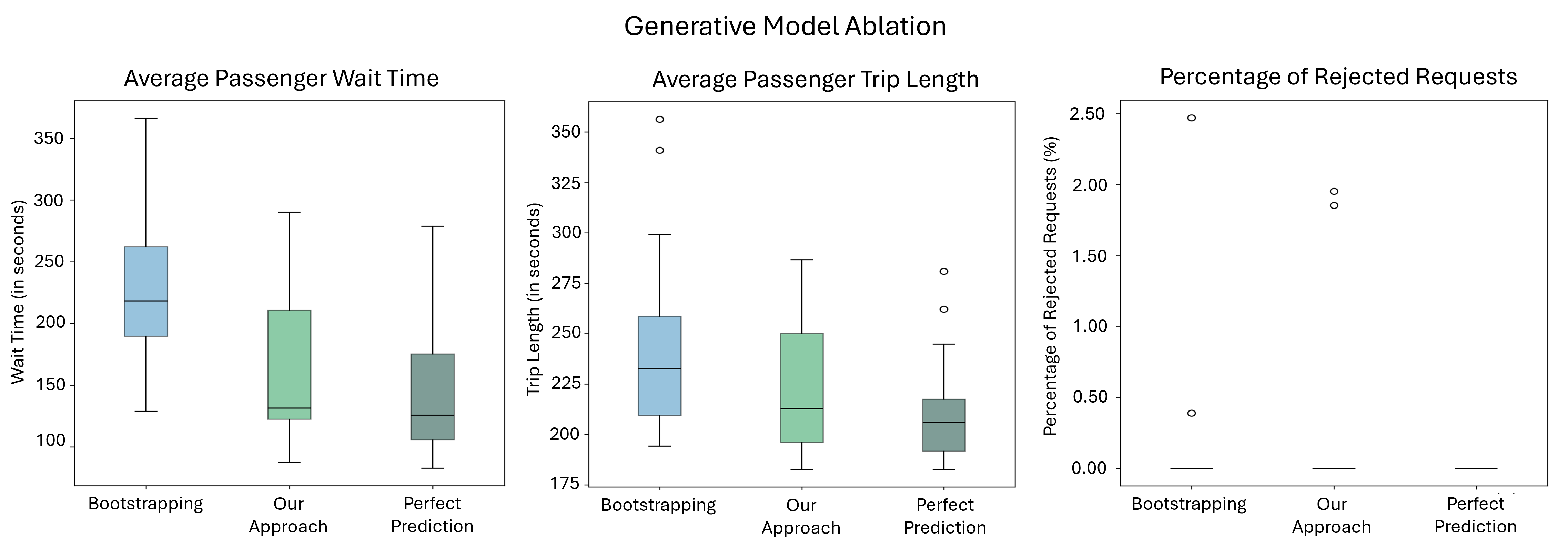}
    \caption{\small{Box plots for the average wait time per request, average trip length per request, and total number of rejected requests for different generative models detailed in Sec.~\ref{subsubsec:ablation_results_generative_model}.}}
    \label{fig:generative_model_ablation}
\end{figure*}

As we can see from Fig.~\ref{fig:generative_model_ablation}, our generative model decreases both the maximum and the median average passenger wait time by approximately $1$ minute and $30$ seconds compared to the bootstrapping model, obtaining wait times that are close to the ones obtained by the perfect prediction model. In terms of the average passenger trip length, our model decreases the spread of the trip lengths, reducing the maximum average trip length by approximately $1$ minute compared to the bootstrapping model. Our generative model, however, obtains longer trip lengths than the perfect prediction model due to predictive errors that are common in deep learning architectures. In terms of the percentage of rejected requests, our model reduces the maximum percentage of rejected requests from $2.5\%$ to $2.0\%$, showing that our method is better at predicting less common sequences of requests, and produces more consistent predictions. However, due to the variability of the bootstrapping process, the bootstrapping method can still sample accurate sequences of requests based exclusively on the bank of historical requests. Even though randomly sampling an accurate sequence of requests has a low probability, it is still possible and can be seen in practice, especially when the request data is seasonal and can be predicted using exclusively historical patterns. For this reason, the bootstrapping model is able to reduce the second largest percentage of rejected requests even beyond what our method is capable of. By looking at the results for the perfect prediction model, we can see that perfect prediction leads to $0$ rejected requests, which illustrates an opportunity for future work to improve the predictive capabilities of the model to further boost performance. In short, all these results show that the time-series transformer provides a predictive advantage over a standard bootstrapping method, but there is still some room for improvement as a perfect prediction model is able to achieve smaller trip lengths and $0$ rejected requests.

\subsubsection{Changing the Promising-Routes Generator}
\label{subsubsec:ablation_results_promising_actions}
To measure the effect of different promising-routes generators, we consider three different ablation cases:
\paragraph{No Swapping} In this scenario, promising-routes are only generated based on assignments of a request to a robot. Swapping of requests between robots is not allowed.
\paragraph{Our approach} This corresponds to our proposed promising-routes generator described in Sec.~\ref{subsec:promising_action_generator} of this paper.
\paragraph{Independent Swapping} This method is inspired by \cite{Wilbur2022}, where only a single request can be swapped between robots. In this method, once the current request gets assigned to a robot $\ell_{m}$, only swaps of requests between other robots that are not robot $\ell_{m}$ are allowed. 

For this ablation study, we report the same three metrics as in the comparison studies (see Sec.~\ref{subsubsec:performance_results}). The results for this ablation are contained in Fig.~\ref{fig:promising_actions_ablation}. 

\begin{figure*}[ht]
    \centering
    \includegraphics[width=\textwidth]{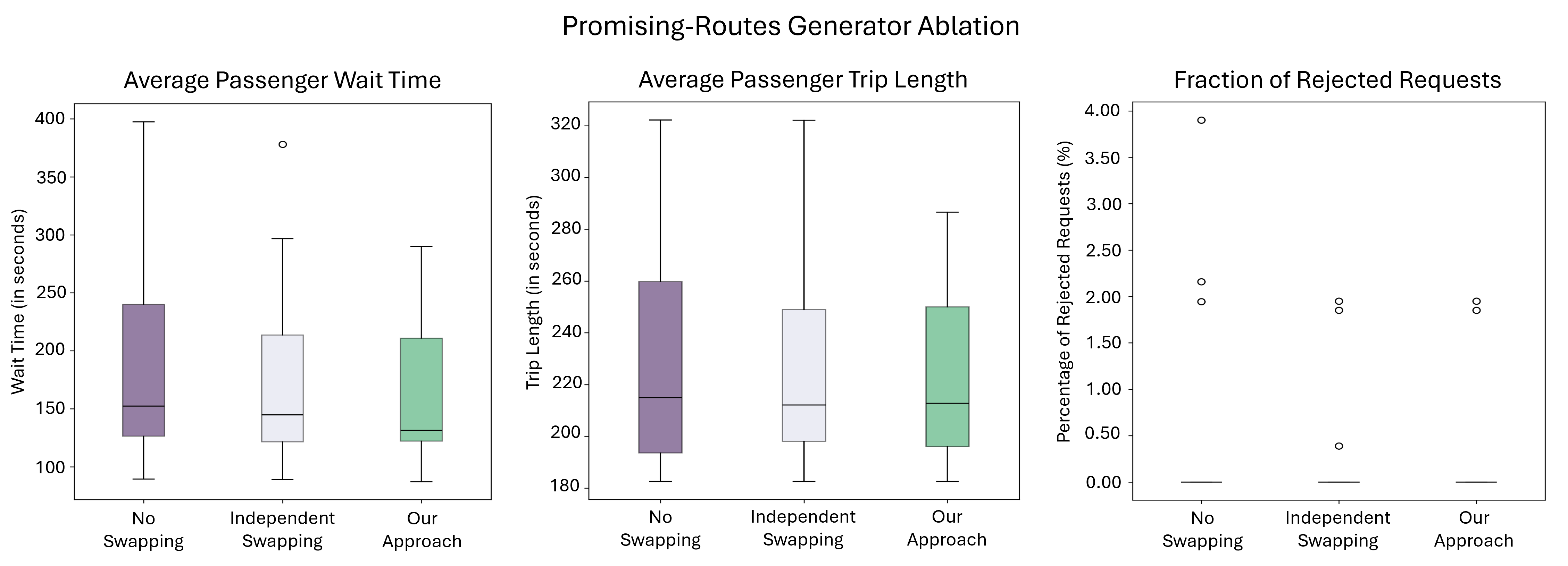}
    \caption{\small{Box plots for average wait time per request, average trip length per request, and total number of rejected requests for different promising-routes generating schemes detailed in Sec.~\ref{subsubsec:ablation_results_promising_actions}.}}
    \label{fig:promising_actions_ablation}
\end{figure*}

As we can see from Fig.~\ref{fig:promising_actions_ablation}, our promising-routes generator results in a decrease of $1$ minute in the maximum average passenger wait time compared to the No Swapping and the Independent Swapping mechanism. Compared to the other methods, our promising-routes generator also reduces the maximum average trip length by $30$ seconds, while achieving the lowest percentage of rejected requests for the outliers, having only two outlier values compared to $3$ outlier values for the other methods. These results show that clustering requests before performing swapping actions create better routes that allow us to service more passengers with lower passenger wait times and lower trip lengths overall.

\section{Conclusion}
In this paper we propose a proactive rollout-based multi-capacity robot routing framework with future request planning that is amenable to theoretical analysis. The framework proposed in this manuscript is applicable to a diverse set of multi-capacity routing problems, including setups in warehouse management, package delivery, and ride-sharing. For this paper, we focus on a transportation setup as an illustrative example, but the framework can be applied to any other multi-capacity autonomous routing setup. As part of our theoretical contributions, we show that our proposed base policy is stable as long as we have a sufficiently large fleet that allows the proposed base policy to service all requests in a finite time horizon. We also propose a fleet-sizing algorithm that provably finds this sufficiently large fleet size. 

To validate our theoretical results and show the advantages of our method, we consider a case study on Harvard's Evening Van data. Using this data, we empirically show that our fleet-sizing algorithm finds a fleet size that allows our base policy and our rollout policy to service all requests in a finite time horizon in practice. Our results also show that our method is able to service around $6\%$ more requests than the closest baseline, decrease median passenger wait times at the station by almost one minute, and maintain similar trip lengths compared to all the other baselines. Our ablation studies show that both the generative model and the unsupervised clustering used in the promising-routes generator provide significant benefits to the route selection procedure. As future work, we want to explore other predictive mechanisms to try to further improve the quality of future cost estimations to achieve results like the ones observed for the oracle in the first ablation study.

\section*{Acknowledgments}
We thank Orhan Eren Akgün for his useful feedback on the theoretical results. We thank Igor Sadalski for his coding tips and template code for the MCTS baseline. We gratefully acknowledge AFOSR award $\#$FA9550-22-1-0223 and DARPA YFA award $\#$D24AP00319-00 for partial support of this work.

\section*{CRedit Autorship Contribution Statement}
\textbf{Daniel Garces}: Conceptualization, Formal analysis, Investigation, Methodology, Software, Validation, Visualization, Writing – original draft. 

\textbf{Stephanie Gil}: Conceptualization, Funding Acquisition, Resources, Supervision, Writing – review and editing.


\section*{Data Availability}
The authors do not have permission to publicly share the data.



\bibliographystyle{elsarticle-num} 
\bibliography{refs}

\end{document}